\documentclass[preprint, ssy]{imsart}

\usepackage{amsthm,amsmath,amssymb}
\RequirePackage[colorlinks,citecolor=blue,urlcolor=blue]{hyperref}

\startlocaldefs

\renewcommand{\d}{\,\mathrm{d}}
\numberwithin{equation}{section}

\theoremstyle{definition}
\newtheorem{setcriteria}{Definition}[section]
\newtheorem{APTdefinition}[setcriteria]{Definition}
\theoremstyle{plain}
\newtheorem{APTtheorem}[setcriteria]{Theorem}

\newtheorem{Noiselemma}{Lemma}[section]
\newtheorem{weakconvergence}[Noiselemma]{Lemma}
\newtheorem{weakconvergencecorol1}[Noiselemma]{Corollary}

\newtheorem{singletimescaletheorem}[Noiselemma]{Theorem}
\newtheorem{singletimescaleconvergence}[Noiselemma]{Corollary}
\newtheorem{MainTheoremLemma}[Noiselemma]{Lemma}

\newtheorem{twotimescaleslemma}{Lemma}[section]
\newtheorem{Lem:TimescalesError}[twotimescaleslemma]{Lemma}
\newtheorem{twotimescalestheorem}[twotimescaleslemma]{Lemma}
\newtheorem{twotimescalescorollary}[twotimescaleslemma]{Corollary}
\theoremstyle{definition}
\newtheorem{JointlyPerturbedDefn}[twotimescaleslemma]{Definition}
\theoremstyle{plain}
\newtheorem{JointlyPerturbedAPTEquiv}[twotimescaleslemma]{Lemma}
\newtheorem{xbarJointlyPerturbedSol}[twotimescaleslemma]{Theorem}
\newtheorem{twotimescaleconvergence}[twotimescaleslemma]{Corollary}
\theoremstyle{remark}
\newtheorem{nonasynchronous}{Remark}
\newtheorem{setvaluedlimit}[nonasynchronous]{Remark}
\theoremstyle{plain}

\theoremstyle{definition}
\newtheorem{LyapunovDfn}{Definition}[section]
\theoremstyle{plain}
\newtheorem{K-BLyapunov}[LyapunovDfn]{Lemma}
\newtheorem{DIConvergence}[LyapunovDfn]{Theorem}
\newtheorem{algorithm2QValueConvvergence}[LyapunovDfn]{Lemma}
\newtheorem{algorithm2BRConvergence}[LyapunovDfn]{Lemma}
\newtheorem{algorithm2convergence}[LyapunovDfn]{Lemma}
\newtheorem{algorithm2Corollary}[LyapunovDfn]{Corollary}

\newtheorem{minupdateproportion}{Lemma}[section]
\newtheorem{Benaimconditions}[minupdateproportion]{Lemma}
\newtheorem{alpharatio}[minupdateproportion]{Lemma}

\endlocaldefs

\begin{document}

\begin{frontmatter}

\title{Asynchronous Stochastic Approximation with Differential Inclusions}
\runtitle{Asynchronous SA with Differential Inclusions}

\begin{aug}
\author{\fnms{Steven} \snm{Perkins}\corref{}\ead[label=e1]{Steven.Perkins.09@bristol.ac.uk}}
\and
\author{\fnms{David S.} \snm{Leslie}\ead[label=e2]{David.Leslie@bristol.ac.uk}}

\affiliation{University of Bristol}

\address{School of Mathematics\\ 
University of Bristol\\
University Walk\\
Clifton\\
Bristol\\
BS8 1TW\\
\printead{e1}\\
\phantom{E-mail:\ }\printead*{e2}}
\end{aug}

\begin{abstract}
The asymptotic pseudo-trajectory approach to stochastic approximation of Bena{\"{\i}}m, Hofbauer and Sorin is extended for asynchronous stochastic approximations with a set-valued mean field. The asynchronicity of the process is incorporated into the mean field to produce convergence results which remain similar to those of an equivalent synchronous process. In addition, this allows many of the restrictive assumptions previously associated with asynchronous stochastic approximation to be removed. The framework is extended for a coupled asynchronous stochastic approximation process with set-valued mean fields. Two-timescales arguments are used here in a similar manner to the original work in this area by Borkar. The applicability of this approach is demonstrated through learning in a Markov decision process. 
\end{abstract}

%\begin{keyword}[class=AMS]
%\kwd[Primary ]{60K35}
%\kwd{60K35}
%\kwd[; secondary ]{60K35}
%\end{keyword}


\end{frontmatter}

%%%%%%%%%%%%%%%%%%%%%%%%%%%%
%%%%%%%%%%%%%%%%%%%%%%%%%%%%
%%%%%%%%%%%%%%%%%%%%%%%%%%%%
%%%%%%%%%%%%%%%%%%%%%%%%%%%%
%%%%%%%%%%%%%%%%%%%%%%%%%%%%
%%%%%%%%%%%%%%%%%%%%%%%%%%%%
%%%%%%%%%%%%%%%%%%%%%%%%%%%%
%%%%%%%%%%%%%%%%%%%%%%%%%%%%
%%%%%%%%%%%%%%%%%%%%%%%%%%%%
%%%%%%%%%%%%%%%%%%%%%%%%%%%%
%%%%%%%%%%%%%%%%%%%%%%%%%%%%
%%%%%%%%%%%%%%%%%%%%%%%%%%%%

\section{Introduction}

Many learning algorithms include a stochastic updating schedule, often based on a Markov chain. Studying the performance of these processes can be carried out using the asynchronous stochastic approximation framework. However, the previous work in this area has focused on continuous, single-valued updates as discussed in the literature (see for example \cite{BorkarAsynchronous}, \cite{KondaBorkar}, \cite{KushnerYinsiam}, \cite{KushnerYinstochastics}, \cite{Tsitsiklis}). Furthermore some of the assumptions which are typically used are challenging to verify. In this work we expand the asymptotic pseudo-trajectory approach of Bena{\"{\i}}m, Hofbauer and Sorin \cite{BenaimHofbauerSorin} to asynchronous stochastic approximations with set-valued mean fields. We incorporate the asynchronicity into the mean field to give a differential inclusion which will characterise the limiting behaviour of the associated learning process.

Consider an iterative process $\{x_n\}_{n \in \mathbb{N}}$ where $x_n \in \mathbb{R}^K$ and denote the $i^{\mathrm{th}}$ component of $x_n$ as $x_n(i)$ where $i \in I$ and $K = |I|$ is finite. A typical stochastic approximation (SA) is of the form

\begin{equation}
	x_{n+1} = x_{n} + \alpha(n+1) \left[F(x_n) + V_{n+1} + d_{n+1} \right], \label{standardSA}
\end{equation}
where $\{\alpha(n)\}_{n \in \mathbb{N}}$ is a positive, decreasing sequence, $\{V_{n}\}_{n \in \mathbb{N}}$ is a zero-mean martingale noise sequence, $\{d_n\}_{n \in \mathbb{N}}$ is a bounded sequence which converges to zero and $F(\cdot): \mathbb{R}^K \rightarrow \mathbb{R}^K$ is a Lipschitz continuous mean field. Standard arguments (e.g. \cite{BenaimHirsch99}) are then used to show that the limiting behaviour of the iterative process in \eqref{standardSA} can be studied through the ordinary differential equation (ODE)

\begin{equation}
	\frac{d x}{dt} = F(x). \label{ODE}
\end{equation}
Commonly known as the ODE method of stochastic approximation, originally proposed by Ljung \cite{Ljung}, this technique has been extended by numerous authors, for example Bena{\"{\i}}m \cite{Benaim}, Bena{\"{\i}}m, Hofbauer and Sorin \cite{BenaimHofbauerSorin}, Borkar \cite{BorkarBook}, Kushner and Clark \cite{KushnerClark} and Kushner and Yin \cite{KushnerYinsiam}, \cite{KushnerYinstochastics}. In particular Bena{\"{\i}}m, Hofbauer and Sorin \cite{BenaimHofbauerSorin} have developed the approach so that under some weak criteria $\{x_n\}_{n \in \mathbb{N}}$ can be updated via a set-valued mean field, $F(\cdot)$. This allows for the limiting behaviour to be studied using the associated differential inclusion. 

Standard stochastic approximations are not always applicable; an example which we examine in this paper is when learning action values in a Markov decision process (MDP) and this is also discussed by Konda and Tsitsiklis \cite{KondaTsitsiklis}, Tsitsiklis \cite{Tsitsiklis} and Singh et al. \cite{SinghJaakkolaLittmanSzepesvari}. In a MDP updates are made to a single random component at each iteration. Therefore we have a stochastic, asynchronous updating pattern, where a subset of an iterative process similar to \eqref{standardSA} can be updated many times before the remaining components are selected for a single update. Based on this idea extensions to the standard theory have been examined such as those by Kushner and Yin \cite{KushnerYinsiam}, \cite{KushnerYinstochastics}. Here however we follow the extension to asynchronous stochastic approximation provided by Borkar \cite{BorkarAsynchronous} and Konda and Borkar \cite{KondaBorkar}. They show that when the iterative updates have a Lipschitz continuous mean field then, similarly to a standard stochastic approximation, the limiting behaviour can be studied via the associated differential equation,

\begin{equation}
	\frac{d x}{dt} = M(t)F(x), \label{asynchronousODE}
\end{equation}
where $M(t)$ is a $K \times K$ diagonal matrix and the diagonal elements of $M(t)$ lie in the set $[0,1]$ for all $t > 0$. This early work on asynchronous stochastic approximations has certain restrictions which limit its usability. In particular, many of the assumptions made in the work of Borkar \cite{BorkarAsynchronous} and Konda and Borkar \cite{KondaBorkar} are given in implicit form and are difficult to verify in specific situations. 

As with the initial results for a standard stochastic approximation the results of Borkar \cite{BorkarAsynchronous} are limited to the case when the mean field, $F(\cdot)$, is a Lipschitz continuous function. The subsequent work by Bena{\"{\i}}m, Hofbauer and Sorin \cite{BenaimHofbauerSorin} on set-valued mean fields leaves the natural question of whether similar results are possible for asynchronous stochastic approximations when a set-valued mean field is used. In addition, the ODE in \eqref{asynchronousODE} is non-autonomous and the scaling matrix $M(\cdot)$ is not explicitly defined. This makes analysis of the limiting behaviour more difficult to study, although some methods for verifying global convergence are outlined by Borkar \cite{BorkarBook}. 

Borkar \cite{BorkarTimescales} originally extended the stochastic approximation framework to two-timescales. Since then Leslie and Collins \cite{LeslieCollins} extended this idea to multiple timescales and Konda and Borkar \cite{KondaBorkar} provide a first venture into the two-timescale asynchronous stochastic approximation. However, all of these only consider stochastic approximations when the mean field is Lipschitz continuous.

The aim of this work is to combine and generalise the results by Borkar \cite{BorkarTimescales}, \cite{BorkarAsynchronous}, Konda and Borkar \cite{KondaBorkar} and Bena{\"{\i}}m, Hofbauer and Sorin \cite{BenaimHofbauerSorin} to create a framework for single and two-timescale asynchronous stochastic approximations which is straightforward to use in practical applications. In this paper we show that, under a set of verifiable assumptions, the diagonal elements of $M(t)$ lie in the closed set $[\varepsilon,1]$, for some $\varepsilon > 0$. The set $[\varepsilon,1]$ can be combined with the mean field $F(\cdot)$ to form a set-valued mean field, $\bar{F}(\cdot)$, whose limiting behaviour can be studied via the associated differential inclusion using the results of Bena{\"{\i}}m, Hofbauer and Sorin \cite{BenaimHofbauerSorin}. A natural benefit of using the differential inclusion framework is that $F(\cdot)$ can be set-valued as this does not alter the analysis.

This paper is organised in the following manner: Section \ref{sec:Notation} reviews some previous results on stochastic approximation with differential inclusions and asynchronous stochastic approximations. In Section \ref{sec:AsynchSA} we focus on the single-timescale asynchronous stochastic approximation. We state the main theorem before presenting the weak convergence results required for the proof. Section \ref{sec:TwoAsynchSA} examines the extension to a two-timescale asynchronous stochastic approximation process. Large parts of this section follow directly from the results in Section \ref{sec:AsynchSA}. In Section \ref{sec:MDPApplication} we present an example of a learning algorithm for discounted reward Markov decision processes and obtain convergence results by applying the method shown in Section \ref{sec:TwoAsynchSA}. This illustrates the ease in which this framework can be used. Finally, the paper concludes with a summary of the work. Throughout this paper many of the proofs are omitted from the main flow of text and are instead presented in an appendix.

%%%%%%%%%%%%%%%%%%%%%%%%%%%%
%%%%%%%%%%%%%%%%%%%%%%%%%%%%
%%%%%%%%%%%%%%%%%%%%%%%%%%%%
%%%%%%%%%%%%%%%%%%%%%%%%%%%%
%%%%%%%%%%%%%%%%%%%%%%%%%%%%
%%%%%%%%%%%%%%%%%%%%%%%%%%%%

\section{Background}
\label{sec:Notation}

Throughout this paper we use two main ideas from the stochastic approximation literature. The first relates to the work by Bena{\"{\i}}m, Hofbauer and Sorin \cite{BenaimHofbauerSorin} on stochastic approximation with differential inclusions, and the second concerns the asynchronous stochastic approximation framework introduced by Borkar \cite{BorkarAsynchronous}. We take the opportunity to review the pertinent features of their work in this section. 

In what is to follow we use the standard concept of set multiplication: if $A \subset \mathbb{R}^{K \times K}$ is a set of $K \times K$ matrices and $B \subset \mathbb{R}^K$ are two closed, convex sets then let the multiplication of these sets be defined as,

\begin{equation*}
	A \cdot B = \big\{a \cdot b ; a \in A, b \in B \big\} \subset \mathbb{R}^K.
\end{equation*}
Note that $A \cdot B$ is also closed and convex. This definition is still used if either or both of the sets $A$ and $B$ are single valued. We also use the same concept when multiplying a constant by a set. That is, if $\alpha$ is a constant then define

\begin{equation*}
	\alpha \cdot B = \big\{\alpha \cdot b ; b \in B \big\} \subset \mathbb{R}^K.
\end{equation*}
However, in this latter case we often drop the `$\cdot$' notation for convenience. 

%%%%%%%%%%%%%%%%%%%%%%%%%%%%
%%%%%%%%%%%%%%%%%%%%%%%%%%%%
%%%%%%%%%%%%%%%%%%%%%%%%%%%%
%%%%%%%%%%%%%%%%%%%%%%%%%%%%
%%%%%%%%%%%%%%%%%%%%%%%%%%%%
%%%%%%%%%%%%%%%%%%%%%%%%%%%%

\subsection{Stochastic Approximation with Differential Inclusions}
\label{sec:BackgroundDI}

We begin by outlining the current convergence results for stochastic approximations with set-valued maps proved by Bena{\"{\i}}m, Hofbauer and Sorin \cite{BenaimHofbauerSorin}. These results are heavily used in Section \ref{sec:AsynchSA}, most notably to prove our main result. Initially we provide a definition which outlines the class of set-valued mean fields we are able to use for stochastic approximation. These criteria are taken directly from the original work on stochastic approximations with differential inclusions by Bena{\"{\i}}m et al. \cite{BenaimHofbauerSorin}.

\begin{setcriteria}
	\label{setcriteria}
	Call $F(\cdot): \mathbb{R}^K \rightarrow \mathbb{R}^K$ a \emph{stochastic approximation map} if it satisfies the following 
	\begin{itemize}
		\item[(i)] $F(\cdot)$ is a closed set-valued map. That is,
		
		\begin{equation*}
			\mbox{Graph}(F) = \left\{ (x,y) ; y \in F(x) \right\},
		\end{equation*}
		is a closed set. Equivalently, $F(\cdot)$ is an upper semi-continuous set-valued map.
		\item[(ii)] For all $x \in \mathbb{R}^{K}$, $F(x)$ is a non-empty, compact, convex subset of $\mathbb{R}^{K}$.
		\item[(iii)] There exists a $c > 0$ such that for all $x \in \mathbb{R}^{K}$,
		\begin{equation*}
			\sup_{z \in F(x)} \left\|z\right\| \leq c \left(1 + \left\|x\right\|\right)
		\end{equation*}
	\end{itemize}
\end{setcriteria}

Take $F(\cdot): \mathbb{R}^K \rightarrow \mathbb{R}^K$ as a stochastic approximation map; a typical differential inclusion is in the form,

\begin{equation}
	\frac{d x}{dt} \in F\big(x\big), \label{DI}
\end{equation}
and a \emph{solution} to \eqref{DI} is an absolutely continuous mapping $x: \mathbb{R} \rightarrow \mathbb{R}^K$ such that $x(0) = x$ and for almost every $t >0$,

\begin{equation*}
	\frac{d x(t)}{dt} \in F\big(x(t)\big).
\end{equation*}
The \emph{flow} induced by \eqref{DI} is defined by,

\begin{equation*}
	\Phi_t(x) = \left\{ x(t) ; x(\cdot) \mbox{ is a solution to \eqref{DI} with } x(0)=x \right\},
\end{equation*}

\begin{APTdefinition}[{Bena{\"{\i}}m et al. \cite{BenaimHofbauerSorin}}]
	\label{APTdefinition}
	A continuous function $x: \mathbb{R}^+ \rightarrow \mathbb{R}^K$ is an \emph{asymptotic pseudo-trajectory} to $\Phi$ if 
	
	\begin{equation*}
		\lim_{t \rightarrow \infty} \sup_{s \in [0,T]} d\Big(x\big(t+s\big),\Phi_s\big(x(t)\big)\Big) = 0,
	\end{equation*}
	for any $T>0$ and where $d(\cdot,\cdot)$ is a distance measure on $\mathbb{R}^K$.
\end{APTdefinition}

Many important properties of a dynamical system $\{\Phi_t(\cdot)\}_{t \geq 0}$ and the asymptotic pseudo-trajectories of the systems are discussed by Bena{\"{\i}}m, Hofbauer and Sorin \cite{BenaimHofbauerSorin}. Most important for the work here is that an asymptotic pseudo-trajectory to \eqref{DI} will behave in a similar manner to the solutions of the differential inclusion and hence the limiting behaviour will be closely related.

We conclude this section by considering a standard iterative process in the form of \eqref{standardSA} where the mean field $F(\cdot): \mathbb{R}^K \rightarrow \mathbb{R}^K$ is a stochastic approximation map. The following theorem states that under four assumptions a linear interpolation of $\{x_n\}_{n \in \mathbb{N}}$ (a function defined precisely in Section \ref{sec:MainResult}) is an asymptotic pseudo-trajectory to the differential inclusion \eqref{DI}. Hence the limiting behaviour of $\{x_n\}_{n \in \mathbb{N}}$ can be studied via this differential inclusion. 

\begin{APTtheorem}
	\label{APTtheorem}
	Assume that 
	\begin{itemize}
		\item[(i)] For all $T > 0$ 
		
		\begin{equation}
			\lim_{n \rightarrow \infty} \sup_k \left\{ \left\| \sum_{i=n}^{k-1} \alpha(i+1)V_{i+1} \right\| ; k = n+1,\ldots,m(\tau_n + T) \right\} \label{kushnerclark}
		\end{equation}
		where $\tau_0 = 0$, $\tau_n = \sum_{i=1}^n \alpha(i)$ and $m(t) = \sup\{k\geq 0 ; t \geq \tau_k\}$,
		\item[(ii)] $\sup_n \|x_n\| = x < \infty$,
		\item[(iii)] $F(\cdot)$ is a stochastic approximation map,
		\item[(iv)] $d_n \rightarrow 0$ as $n \rightarrow \infty$ and $\sup_n \|d_n\| = d < \infty$.
	\end{itemize}
	Then a linear interpolation of the iterative process $\{x_n\}_{n \in \mathbb{N}}$ given by \eqref{standardSA} is an asymptotic pseudo-trajectory of the differential inclusion \eqref{DI}.
\end{APTtheorem}

This is a slight modification to a result stated by Bena{\"{\i}}m, Hofbauer and Sorin \cite[Proposition 1.3]{BenaimHofbauerSorin} to include the $\{d_n\}_{n \in \mathbb{N}}$ terms. It is trivial to verify that this will not alter any of the asymptotic results of the original work. 

%%%%%%%%%%%%%%%%%%%
%%%%%%%%%%%%%%%%%%%
%%%%%%%%%%%%%%%%%%%
%%%%%%%%%%%%%%%%%%%
%%%%%%%%%%%%%%%%%%%
%%%%%%%%%%%%%%%%%%%

\subsection{Asynchronous Stochastic Approximations}
\label{sec:BackgroundAsynch}

Now we fully introduce the asynchronous stochastic approximation notation used. A typical asynchronous stochastic approximation such as those studied by Borkar \cite{BorkarAsynchronous} fits the following framework. If $2^I$ is the power set of all possible updating combinations in $I$ then let $\bar{I}_n \in 2^I$ be the components of the iterative process $\{x_n\}_{n \in \mathbb{N}}$ updated at iteration $n$. Using a counter for state $i \in I$,

\begin{equation*}
	\nu_n(i) := \sum_{k=1}^n \mathbb{I}_{\{i \in \bar{I}_k\}},
\end{equation*}
we consider processes in which no component, $i$, in the asynchronous process needs to know the global counter, $n$, merely its own counter, $\nu_n(i)$. Let $F_i(\cdot)$, $x_n(i)$, $V_n(i)$ and $d_n(i)$ be the $i^{\mathrm{th}}$ component of $F(\cdot)$, $x_n$, $V_n$ and $d_n$ respectively, for $i = 1, \ldots, K$. We directly extend the notation used in \eqref{standardSA} for an asynchronous stochastic approximation; let $F(\cdot)$ be a stochastic approximation map, then for $i = 1,\ldots,K$ let

\begin{equation}
	x_{n+1}(i) \in x_{n}(i) + \alpha \big(\nu_{n+1}(i)\big) \mathbb{I}_{\{i \in \bar{I}_{n+1}\}} \big[F_i(x_n) + V_{n+1}(i) + d_{n+1}(i) \big]. \label{asynchronousSA}
\end{equation}

Define the \emph{asynchronous step sizes}, $\bar{\alpha}_n$, and the \emph{relative step sizes}, $\mu_n(i)$, to be

\begin{equation*}
	\bar{\alpha}_n := \max_{i \in \bar{I}_{n}} \alpha \big(\nu_{n}(i)\big), \quad \mu_n(i) := \frac{\alpha\big(\nu_{n}(i)\big) }{\bar{\alpha}_n} \mathbb{I}_{\{i \in \bar{I}_{n}\}}.
\end{equation*}
The asynchronous step sizes, $\bar{\alpha}_n$, are random step sizes (in contrast with the deterministic $\alpha(n)$ terms) whilst the relative step sizes, $\mu_n(i)$ are zero whenever the $i^{\mathrm{th}}$ component of the iterative process is not updated. Clearly $\mu_n(i) \in [0,1]$. By letting $M_n$ be the $K \times K$ diagonal matrix of the $\mu_n(i)$ terms we can express the previous asynchronous stochastic approximation \eqref{asynchronousSA} in the more concise form

\begin{equation}
	x_{n+1} - x_{n} - \bar{\alpha}_{n+1} M_{n+1} \left[V_{n+1} + d_{n+1} \right] \in \bar{\alpha}_{n+1} M_{n+1} \cdot F(x_n). \label{asynchronousSAmatrix} 
\end{equation}
This is a more familiar form for a stochastic approximation with a set-valued mean field. If $F(\cdot)$ is a stochastic approximation map in \eqref{standardSA} then \eqref{asynchronousSAmatrix} differs from \eqref{standardSA} only in that the step sizes in \eqref{asynchronousSAmatrix} are random and the addition of the $M_{n+1}$ coefficient. Instead of thinking of $M_{n+1}$ as a coefficient of the step sizes we combine it with the mean field. Convergence of the error term, $d_{n+1}$, will be unaffected and, under a set of assumptions in Section \ref{sec:SingleAssumptions}, the noise term $M_{n+1}V_{n+1}$ will still satisfy the Kushner and Clark noise condition \eqref{kushnerclark}. Combining the $M_{n+1}$ term and the mean field term $F(\cdot)$ into a single set provides an intuitive method of rephrasing the stochastic approximation and leads to a set-valued mean field. 

Proceeding with this intuition is not immediately straightforward since $M_n$ is time varying and can be zero infinitely often, in which case the mean field could be zero even when the original update term, $F(\cdot)$, is not. This would mean the limiting behaviour of the differential inclusion in the asynchronous stochastic approximation could be different to the synchronous case, where ultimately we wish to say that the two behave in the same manner in the limit. To avoid this scenario we follow Borkar \cite{BorkarAsynchronous} and consider the weak limit of the interpolations of $\{M_n\}_{n \in \mathbb{N}}$, which will always be bounded away from zero under some verifiable assumptions, given in Section \ref{sec:SingleAssumptions}.

%%%%%%%%%%%%%%%%%%%%%%%%%%%%
%%%%%%%%%%%%%%%%%%%%%%%%%%%%
%%%%%%%%%%%%%%%%%%%%%%%%%%%%
%%%%%%%%%%%%%%%%%%%%%%%%%%%%
%%%%%%%%%%%%%%%%%%%%%%%%%%%%
%%%%%%%%%%%%%%%%%%%%%%%%%%%%

\section{Asynchronous SA with Differential Inclusions}
\label{sec:AsynchSA}

We begin by presenting the main result of this paper which concerns the limiting behaviour of the asynchronous stochastic approximation in \eqref{asynchronousSAmatrix}, before outlining the results required to prove this in the remainder of the section.

%%%%%%%%%%%%%%%%%%%%%%%%%%%%
%%%%%%%%%%%%%%%%%%%%%%%%%%%%
%%%%%%%%%%%%%%%%%%%%%%%%%%%%
%%%%%%%%%%%%%%%%%%%%%%%%%%%%
%%%%%%%%%%%%%%%%%%%%%%%%%%%%
%%%%%%%%%%%%%%%%%%%%%%%%%%%%

\subsection{Main Result}
\label{sec:MainResult}

Assume that $F(\cdot)$ is a stochastic approximation map and for all $n$ define $f_n \in F(x_n)$ by its component parts, $f_n(i)$, $i = 1,\ldots,K$, such that

\begin{equation}
	f_n(i) \mu_{n+1}(i) := \frac{x_{n+1}(i) - x_n(i)}{\bar{\alpha}_{n+1}} - \mu_{n+1}(i)\left[ V_{n+1}(i) + d_{n+1}(i) \right].\label{trueupdate}
\end{equation}

Notice that if $\mu_{n+1}(i) = 0$ then we can select any $f_n(i) \in F_i(x_n)$. Then we can write the iterative process in \eqref{asynchronousSAmatrix} as

\begin{equation}
	x_{n+1} = x_n + \bar{\alpha}_{n+1} M_{n+1} \Big[ f_n + V_{n+1} + d_{n+1} \Big]. \label{asynchronousSAmatrix2}
\end{equation}
For some fixed $\varepsilon > 0$, let $\tilde{M}_n$ be a series of $K \times K$ diagonal matrices with entries in the set $[\varepsilon, 1]$, for all $n$, to be defined in Section \ref{section:AsynchronousUpdates}. We can again rewrite the iterative process in \eqref{asynchronousSAmatrix2} as
	
\begin{equation*}
	x_{n+1} = x_n + \bar{\alpha}_{n+1} \Big[\tilde{M}_{n+1} f_n  + \big( M_{n+1} - \tilde{M}_{n+1} \big) f_n + M_{n+1} V_{n+1} + M_{n+1} d_{n+1} \Big].
\end{equation*}
Now by letting $\bar{V}_{n+1} = f_n \big( M_{n+1} - \tilde{M}_{n+1} \big) + M_{n+1} V_{n+1}$ and $\bar{d}_{n+1} = M_{n+1} d_{n+1}$ we get,

\begin{equation}
	x_{n+1} = x_n + \bar{\alpha}_{n+1} \Big[ \tilde{M}_{n+1} f_n + \bar{V}_{n+1} + \bar{d}_{n+1} \Big]. \label{modifiedx}
\end{equation}
For general $k, \delta > 0$ let 

\begin{equation}
	\Omega^{\delta}_{k} := \big\{ \mathrm{diag}(\omega_1,...,\omega_{k}); \omega_i \in [\delta,1], \forall i = 1,...,k \big\}, \label{OmegaSetdfn}
\end{equation}
and define 

\begin{equation}
	\bar{F}(x) := \Omega^{\varepsilon}_{K} \cdot F(x). \label{Fbar}
\end{equation}
If $F(\cdot)$ is Lipschitz continuous direct comparisons can be made between the mean field, $\bar{F}(x)$, and the analogous mean field $M(t)F(x)$ from \eqref{asynchronousODE} which is used by Borkar \cite{BorkarAsynchronous} and Konda and Borkar \cite{KondaBorkar}. This provides the key insight into the new approach we take. Under the assumptions used in Section \ref{sec:SingleAssumptions} the equivalent $M(t)$ values almost surely lie in $\Omega^{\varepsilon}_{K}$. By combining this with $F(x)$ we produce a differential inclusion which is more straightforward to study than a non-autonomous differential equation and naturally fits the stochastic approximation framework of Bena{\"{\i}}m, Hofbauer and Sorin \cite{BenaimHofbauerSorin}. In addition, this idea naturally lends itself to examining a similar process for a set-valued mean field as we proceed to do in this paper. 

Equation \eqref{modifiedx} can be expressed in the form of a stochastic approximation with a set-valued mean field as in \cite{BenaimHofbauerSorin}:

\begin{equation}
	x_{n+1} - x_n - \bar{\alpha}_{n+1} \bar{V}_{n+1} - \bar{\alpha}_{n+1} \bar{d}_{n+1} \in \bar{\alpha}_{n+1} \bar{F}(x_n). \label{xinclusion}
\end{equation}
Let $\bar{\tau}_0 := 0$, $\bar{\tau}_n := \sum_{k=1}^{n} \bar{\alpha}_k$ be the timescale for the asynchronous updates. To allow this process to be analysed in continuous time consider an interpolated version of the stochastic approximation \eqref{xinclusion} so that this process can be considered in continuous time, 

\begin{equation}
	\bar{x}(\bar{\tau}_n+s) = x_n + s \frac{x_{n+1} - x_n}{\bar{\alpha}_{n+1}}, \quad s \in \left[0,\bar{\alpha}_{n+1}\right). \label{interpolatedsA}
\end{equation}
Under the assumptions (A1)-(A5), presented in Section \ref{sec:SingleAssumptions}, we show in Section \ref{section:AsynchronousUpdates} that a sequence, $\{\tilde{M}_n\}_{n \in \mathbb{N}}$, can be defined such that $\{\bar{\alpha}_n\}_{n \in \mathbb{N}}$ and $\{\bar{V}_{n}\}_{n \in \mathbb{N}}$ satisfy the Kushner and Clark noise condition in \eqref{kushnerclark}. By invoking Theorem \ref{APTtheorem} we obtain our main result, which is proved in Section \ref{sec:MainResultProof}.

\begin{singletimescaletheorem}
	\label{singletimescaletheorem}
	Under the assumptions (A1)-(A5), with probability 1, $\bar{x}(t)$ is an asymptotic pseudo-trajectory to the differential inclusion,
	
	\begin{equation}
		\frac{d x}{d t} \in \bar{F}(x). \label{asynchronousinclusion}
	\end{equation}
\end{singletimescaletheorem}

Directly from Theorem \ref{singletimescaletheorem} and \cite[Proposition 3.27]{BenaimHofbauerSorin} we get the key result concerning the convergence of an asynchronous stochastic approximation process.

\begin{singletimescaleconvergence}
	\label{singletimescaleconvergence}
	If there is a globally attracting set, $A$, for the differential inclusion \eqref{asynchronousinclusion}, and assumptions (A1)-(A5) are satisfied, then the iterative process \eqref{asynchronousSAmatrix} will converge to $A$.
\end{singletimescaleconvergence}

%%%%%%%%%%%%%%%%%%%%%%%%%%%%
%%%%%%%%%%%%%%%%%%%%%%%%%%%%
%%%%%%%%%%%%%%%%%%%%%%%%%%%%
%%%%%%%%%%%%%%%%%%%%%%%%%%%%
%%%%%%%%%%%%%%%%%%%%%%%%%%%%
%%%%%%%%%%%%%%%%%%%%%%%%%%%%

\subsection{Assumptions}
\label{sec:SingleAssumptions}

Throughout this section we study the convergence properties of the iterative process \eqref{asynchronousSAmatrix}. We make reference to the following assumptions, (A1)-(A5), all of which are either standard requirements for a stochastic approximation or can be verified prior to running the asynchronous stochastic approximation process. This is in contrast with the previous work on asynchronous stochastic approximations by Borkar \cite{BorkarAsynchronous} and Konda and Borkar \cite{KondaBorkar}.

\begin{itemize}
	\item[(A1)] 
	\begin{itemize}
		\item[(a)] For a compact set, $C \subset \mathbb{R}^K$, $x_n \in C$ for all $n$.
		\item[(b)] $\{d_n\}_{n \in \mathbb{N}}$ is a bounded sequence such that $d_n \rightarrow 0$ as $n \rightarrow \infty$.
	\end{itemize}
	\item[(A2)] Let $\alpha(n)$ satisfy the following criteria,
	\begin{itemize}
		\item[(a)] $ \sum_n \alpha(n) = \infty$ and $\alpha(n) \rightarrow 0$ as $n \rightarrow \infty$,
		\item[(b)] For $x \in (0,1)$, $\sup_n \alpha([xn]) / \alpha(n) < A_x < \infty$, where $[\cdots]$, means the ``integer part of''. In addition, for all $n$, $\alpha(n) \geq \alpha(n+1)$.
	\end{itemize}
	\item[(A3)] $F(\cdot)$ is a stochastic approximation map.
\end{itemize}

(A1)(a) is a slight strengthening of the standard stochastic approximation boundedness assumption; however this is still a relatively mild condition. Methods to ensure that it is satisfied are discussed elsewhere, for example \cite{Kamal}, \cite{KondaBorkar} or \cite{Tsitsiklis}. A basic restriction is placed on $\{d_n\}_{n \in \mathbb{N}}$ in (A1)(b); in this form the sequence does not affect the asymptotic behaviour of the process. (A2)(a) is a standard assumption required for stochastic approximation, and (A2)(b) is a mild technical condition required to deal with the asynchronicity, which is also used by Borkar \cite{BorkarAsynchronous}. We have dropped the additional restriction on the step-sizes used by Borkar which severely restricts the possible choices of $\{\alpha(n)\}_{n \in \mathbb{N}}$. (A3) ensures that we can use the convergence results presented in Theorem \ref{APTtheorem} and is a standard assumption for stochastic approximations with a set-valued mean field.

Define $\bar{I} \subset 2^I$ as the set of all the possible combinations which have positive probability of occurring. As an example, if every element of $I$ gets updated and it is known that $\bar{I}_n$ is a singleton for each $n$, then $\bar{I} = I$. 

Let $\mathcal{F}_{n}$ be a sigma algebra containing all the information up to and including the $n^{\mathrm{th}}$ iteration. That is $\mathcal{F}_{n} := \sigma(\{\bar{I}_m\}_m,\{x_m\}_m,\{\nu_m(i)\}_{i,m};\forall m \leq n, i = 1,\ldots,K)$.

\begin{itemize}
	\item[(A4)] 
	\begin{itemize}
		\item[(a)] For all $x \in C$, $\mathcal{I}_n, \mathcal{I}_{n+1} \in \bar{I}$,
			
		\begin{equation*}
			\mathbb{P}\Big(\bar{I}_{n+1}=\mathcal{I}_{n+1} | \mathcal{F}_{n}\Big) = \mathbb{P}\Big(\bar{I}_{n+1}=\mathcal{I}_{n+1} |\bar{I}_n=\mathcal{I}_n,x_n=x\Big).
		\end{equation*}
	\end{itemize}
	Let
	\begin{equation*}
		P_{(\mathcal{I}_n, \mathcal{I}_{n+1})}(x):= \mathbb{P}\Big(\bar{I}_{n+1}=\mathcal{I}_{n+1} |\bar{I}_n=\mathcal{I}_n,x_n=x\Big).
	\end{equation*}
	\begin{itemize}
		\item[(b)] For all $x \in C$ the transition probabilities $P_{(\mathcal{I}_n, \mathcal{I}_{n+1})}(x)$ form an aperiodic, irreducible, positive recurrent Markov chain over $\bar{I}$ and for all $i \in I$ there exists an $\mathcal{I} \in \bar{I}$ such that $i \in \mathcal{I}$.
		\item[(c)] The map $x \mapsto P_{(\mathcal{I}_n, \mathcal{I}_{n+1})}(x)$ is Lipschitz continuous.
	\end{itemize}
\end{itemize}

(A4)(a) assumes that the transitions between the updated elements in $\bar{I}$ are part of a controlled Markov chain. (A4)(b) is a straightforward assumption on this controlled Markov chain which can be verified prior to implementation which allows us to negate the need for some of the original technical assumptions made by Konda and Borkar \cite{KondaBorkar}. In this previous work Konda and Borkar assume that every state is updated at a comparable rate in the limit which cannot directly be verified prior to running the process. (A4)(c) is a condition which is required later to use a result from Ma et al. \cite{MaMakowskiShwartz} on the convergence of stochastic approximation with Markovian Noise.

\begin{itemize}
	\item[(A5)]
	
	\begin{itemize}
		\item[(a)] For some $q \geq 2$
			
			\begin{equation*}
				\sum_n \alpha(n)^{1 + q/2} < \infty, \quad \mbox{and} \quad \sup_n \mathbb{E}\big(\|V_n \|^q\big) < \infty;
			\end{equation*}
	\end{itemize}
	\begin{itemize}
		\item[(b)] Take $V_n(i)$ independent of $V_n(j)$ for $i \neq j$ and $V_n(i)$ independent of $\bar{I}_n$ given $\mathcal{F}_{n-1}$ for all $i = 1,\ldots,K$. Let $\langle a,b \rangle = \sum_k a_k b_k$. Then there exists a positive $\Gamma$ such that for all $\theta \in \mathbb{R}^K$,
		
			\begin{equation*}
				\sum_n e^{-c/\alpha(n)} < \infty,
			\end{equation*}
			and
			\begin{equation*}
				\mathbb{E}\Big[\exp\Big\{\langle \theta,V_{n+1} \rangle \Big\} | \mathcal{F}_n\Big] \leq \exp\Big\{\frac{\Gamma}{2} \|\theta\|^2 \Big\};
			\end{equation*}
			for each $c > 0$.
	\end{itemize}
		
	We say that (A5) holds if \emph{either} (A5)(a) or (A5)(b) is true.
\end{itemize}

An assumption similar to (A5) is used by Bena{\"{\i}}m, Hofbauer and Sorin \cite{BenaimHofbauerSorin} to verify a condition for noise convergence and is similar to that used by Kushner and Clark \cite{KushnerClark}. We use this assumption only to show the noise term still satisfies the Kushner-Clark condition with the convergence given in Lemma \ref{lemma:Noise}; the proof is presented via two lemmas in Appendix \ref{Appendix:Noise}. 

\begin{Noiselemma}
	\label{lemma:Noise}
	Assume that (A2)(b), (A4) and (A5) hold. Then with probability 1, for all $T>0$,
	
		\begin{equation*}
			\lim_{n \rightarrow \infty} \sup \bigg\{ \Big\| \sum_{i=n}^{k-1} \bar{\alpha}_{i+1} M_{i+1} V_{i+1} \Big\| ; k = n+1, \dots,\bar{m}(\bar{\tau}_n + T) \bigg\} = 0 
		\end{equation*}
		where $\bar{\tau}_0 := 0$, $\bar{\tau}_n := \sum_{k=1}^{n} \bar{\alpha}_k$ and $\bar{m}(t) := \sup \{k \geq 0 ; t \geq \bar{\tau}_k\}$.
\end{Noiselemma}

Note that if Lemma \ref{lemma:Noise} can be verified directly without (A5) then this assumption is redundant, and hence we only require (A1)-(A4) and Lemma \ref{lemma:Noise} to hold. This approach is used in Section \ref{sec:TwoAsynchSA}.

%%%%%%%%%%%%%%%%%%%%%%%%%%%%
%%%%%%%%%%%%%%%%%%%%%%%%%%%%
%%%%%%%%%%%%%%%%%%%%%%%%%%%%
%%%%%%%%%%%%%%%%%%%%%%%%%%%%
%%%%%%%%%%%%%%%%%%%%%%%%%%%%
%%%%%%%%%%%%%%%%%%%%%%%%%%%%

\subsection{Weak Convergence of Asynchronous Updates}
\label{section:AsynchronousUpdates}

As discussed in the introduction, the key issue with asynchronous stochastic approximations is how to handle the interaction of the relative step sizes and the mean field. It is important to be able to bound the limit of the relative step sizes, $\mu_n(i)$, away from zero for all $i \in I$ in order to produce an asynchronous stochastic approximation mean field, $\bar{F}(\cdot)$ which will behave similarly to the synchronous mean field, $F(\cdot)$. However the relative step size of a state is zero whenever that state is not updated, hence it is not immediately clear that this is even possible. Despite this it is sufficient that for any $T>0$ an `average' of $\mu_n(i)$ over length $T$ in the continuous time interpolation converges to a value which is bounded away from zero. In this section we prove that under (A2)(b) and (A4) this is indeed the case.

For $T>0$ the space $L^2([0,T])$ is the set of measurable functions $h(\cdot): \mathbb{R} \rightarrow \mathbb{R}$ such that,

\begin{equation*}
	\Big( \int_{0}^T |h(s)|^2 \d s \Big)^{\frac{1}{2}} < \infty.
\end{equation*}
Following the method used in \cite{BorkarBook} and \cite{KondaBorkar}, define $\mathcal{U}$ to be the space of maps $u(\cdot) : \mathbb{R} \rightarrow [0,1]$ with the coarsest topology which for all $T>0$ leaves continuous the map,

\begin{equation*}
	u(\cdot) \rightarrow \int_0^T h(s)u(s) \d s
\end{equation*}
for all $h(\cdot) \in L^2([0,T])$. Hence $\mathcal{U}$ is a space of $[0,1]$ trajectories. This means that for any map defined on $\mathcal{U}$ convergence to a limit point will be in the weak sense, along a subsequence. That is a sequence of maps $\{\varphi_n(\cdot)\}_{n \in \mathbb{N}}$ such that $\varphi_n(\cdot) \in \mathcal{U}$ for all $n$ is said to possess a limit point $\tilde{\varphi}(\cdot) \in \mathcal{U}$ if for fixed $T >0$ there exists a subsequence $k(n)$, such that for any $h(\cdot) \in L^2([0,T])$,

\begin{equation}
	\int_0^T  h(s)\varphi_{k(n)}(s) \d s \rightarrow \int_0^T h(s)\tilde{\varphi}(s) \d s, \quad \mbox{as } n \rightarrow \infty. \label{Dfn:weakconvergence}
\end{equation}
Many authors provide a more detailed discussion on weak convergence; for example \cite{Billingsley}, \cite[Appendix A]{BorkarBook} or  \cite{DebnathMikusinski}.

Now we extend the relative step sizes, $\mu_n$, to continuous time; for all $i \in I$ let $u_i(t) = \mu_{n+1}(i)$ for $t \in [\bar{\tau}_n,\bar{\tau}_{n+1})$ and let $u(\cdot) = \big(u_1(\cdot),\ldots,u_K(\cdot)\big)$. For all $i \in I$ and $t > 0$ define $\bar{u}_i^n(t) := u_i(t+\bar{\tau}_n) \in \mathcal{U}$. 

\begin{weakconvergence}
	\label{weakconvergence}
	Under (A2)(b) and (A4), for all $i \in I$ and for any $T > 0$, $\{\bar{u}_i^n(\cdot)\}_{n \in \mathbb{N}}$ converges along a subsequence to a limit point $\tilde{u}_i(\cdot)$ such that for some $\varepsilon > 0$ and any $0< v \leq T$,
		
	\begin{equation}
		\int_0^v \tilde{u}_i(s) \d s \geq v \varepsilon, \quad \mbox{a.s.}. \label{weakconvergencecriteria}
	\end{equation}
\end{weakconvergence}

\begin{proof}
	See Appendix \ref{appendix:uvconvergence}.
\end{proof}

\begin{weakconvergencecorol1}
	\label{weakconvergencecorol1}
	For any $T>0$ let $\tilde{u}(\cdot)$ be a limit point of $\{\bar{u}^n(\cdot)\}_{n \in \mathbb{N}}$ in $\mathcal{U}$, then under (A2)(b) and (A4) there exists an $\varepsilon > 0$ such that for all $i \in I$ and any $t,v >0$ such that $t+v \leq T$,
	
	\begin{equation*}
		\int_t^{t+v} \tilde{u}_i(s) \d s \geq v \varepsilon, \quad \mbox{a.s.}.
	\end{equation*}
\end{weakconvergencecorol1}

\begin{proof}
	See Appendix \ref{appendix:uvconvergence}.
\end{proof}

We now expand upon the discussions in Sections \ref{sec:BackgroundAsynch} and \ref{sec:MainResult} on producing a sequence of matrices $\{\tilde{M}_n\}_{n \in \mathbb{N}}$. In order to use the differential inclusions framework described in Section \ref{sec:BackgroundDI} we need to define a sequence of diagonal matrices, $\{\tilde{M}_n\}_{n \in \mathbb{N}}$ with diagonal entries which are always in the set $[\varepsilon, 1]$, for some $\varepsilon >0$, and such that the terms converge to the same limit as the terms of $\{M_n\}_{n \in \mathbb{N}}$. Recall that $M_n$ is a diagonal matrix containing the $\mu_n(i)$ terms.

Fix $\varepsilon >0$ taken from Lemma \ref{weakconvergence} and define a new function $v(\cdot): \mathbb{R} \rightarrow \mathbb{R}^K$ such that

\begin{equation*}
	v_i(t) := \max\big\{u_i(t),\varepsilon\big\}.
\end{equation*}
For all $t > 0$ let $\bar{v}_i^n(t) := v_i(t+ \bar{\tau}_n)$. Corollary \ref{weakconvergencecorol1} shows that, with respect to the topology of $\mathcal{U}$, in the limit $u_i(t) \in [\varepsilon,1]$ for almost every $t$ and similarly $v_i(t) \in [\varepsilon,1]$ for all $t$. From this it is clear $u_i(t)$ and $v_i(t)$ have the same limit point in $\mathcal{U}$. That is, if $\tilde{u}(t)$ is a limit point of $\{\bar{u}^n(\cdot)\}_{n \in \mathbb{N}}$ then it is also a limit point of $\{\bar{v}^n(\cdot)\}_{n \in \mathbb{N}}$. Hence for any $T>0$ there exists a subsequence $k(n)$ such that for $h(\cdot) \in L^2([0,T])$,

\begin{equation*}
	\int_{0}^{T} \tilde{u}_i(t) h(t) \d t = \lim_{n \rightarrow \infty} \int_{0}^{T} \bar{u}_i^{k(n)}(t) h(t) \d t = \lim_{n \rightarrow \infty} \int_{0}^{T} \bar{v}_i^{k(n)}(t) h(t) \d t.
\end{equation*}
However, the key interest here is in the convergence of $u_i(\cdot)$ and $v_i(\cdot)$. Following the reasoning of Borkar \cite{BorkarAsynchronous} and Konda and Borkar \cite{KondaBorkar}, it does not matter whether $\bar{u}_i^n(\cdot)$ and $\bar{v}_i^n(\cdot)$ converge directly or via a subsequence as this does not affect the convergence of the continuous processes $u_i(\cdot)$ and $v_i(\cdot)$. Hence we can say that $u_i(\bar{\tau}_n + \cdot), v_i(\bar{\tau}_n + \cdot)$ converge weakly to a limit point $\tilde{u}_i(\cdot)$, or equivalently, if $h(\cdot)$ is any bounded, continuous function then for all $i = 1,\ldots,K$,

\begin{equation}
	\lim_{n \rightarrow \infty} \left\| \int_{\bar{\tau}_n}^{\bar{\tau}_n+v} \big[v_i(s) - u_i(s)\big] h(s) \d s\right\| = 0. \label{mtilderesult}
\end{equation}
Define $\tilde{M}(t)$ to be the $K \times K$ diagonal matrix of the $v_i(t)$ terms and let $\tilde{M}_n := \tilde{M}(\bar{\tau}_n)$.

\begin{MainTheoremLemma}
	\label{MainTheoremLemma}
	Almost surely under assumptions (A2)(b) and (A4),
	
	\begin{equation*}
		\sup_k \left\{\left\| \sum_{i=n}^{k-1} \bar{\alpha}_{i+1} f_i \big( M_{i+1} - \tilde{M}_{i+1} \big) \right\| ; k = n+1,\ldots \bar{m}(\bar{\tau}_n + T) \right\} = 0.
	\end{equation*}
\end{MainTheoremLemma}

\begin{proof}
	See Appendix \ref{K-CAsynchNoiseLemmaProof}. The proof relies on \eqref{mtilderesult}. 
\end{proof}

%%%%%%%%%%%%%%%%%%%%%%%%%%%%
%%%%%%%%%%%%%%%%%%%%%%%%%%%%
%%%%%%%%%%%%%%%%%%%%%%%%%%%%
%%%%%%%%%%%%%%%%%%%%%%%%%%%%
%%%%%%%%%%%%%%%%%%%%%%%%%%%%
%%%%%%%%%%%%%%%%%%%%%%%%%%%%

\subsection{Proof of Theorem \ref{singletimescaletheorem}}
\label{sec:MainResultProof}

We must verify that the four conditions of Theorem \ref{APTtheorem} hold for the stochastic approximation process in \eqref{xinclusion} to ascertain that $\bar{x}(t)$ is an asymptotic pseudo-trajectory of \eqref{asynchronousinclusion}. 
	
Fix $T>0$. Then,
	
\begin{align}
	\sup_k & \left\{ \Big\| \sum_{i=n}^{k-1} \bar{\alpha}_{i+1} \bar{V}_{i+1} \Big\| ; k = n+1,\ldots \bar{m}(\bar{\tau}_n + T) \right\} \notag \\ 
	& \leq \sup_k \left\{\left\| \sum_{i=n}^{k-1} \bar{\alpha}_{i+1} M_{i+1} V_{i+1} \right\|; k = n+1,\ldots \bar{m}(\bar{\tau}_n + T) \right\} \label{noiseI} \\
	& + \sup_k \left\{\left\| \sum_{i=n}^{k-1} \bar{\alpha}_{i+1} f_i \big( M_{i+1} - \tilde{M}_{i+1} \big) \right\| ; k = n+1,\ldots \bar{m}(\bar{\tau}_n + T) \right\}. \label{asynchII} 
\end{align}	
Using Lemma \ref{lemma:Noise} and Lemma \ref{MainTheoremLemma} immediately gives that \eqref{noiseI} and \eqref{asynchII} converge to zero a.s., and hence this verifies that property $(i)$ holds. Assumption (A1)(a) directly gives that $(ii)$ holds. Lastly, it is straightforward to verify that, under (A1)-(A5), $\bar{F}(\cdot)$ is a stochastic approximation map which verifies condition $(iii)$, and (A1)(b) is equivalent to $(iv)$. \hfill $\square$

%%%%%%%%%%%%%%%%%%%%%%%%%%%%
%%%%%%%%%%%%%%%%%%%%%%%%%%%%
%%%%%%%%%%%%%%%%%%%%%%%%%%%%
%%%%%%%%%%%%%%%%%%%%%%%%%%%%
%%%%%%%%%%%%%%%%%%%%%%%%%%%%
%%%%%%%%%%%%%%%%%%%%%%%%%%%%

\section{Two-timescale Asynchronous Stochastic Approximation}
\label{sec:TwoAsynchSA}

A useful extension of standard stochastic approximations is to two-timescales. This concept was originally introduced by Borkar \cite{BorkarTimescales} and has later been used by Leslie and Collins \cite{LeslieCollins} for multiple timescales and Konda and Borkar \cite{KondaBorkar} for two-timescales asynchronous stochastic approximation. If we have a coupled pair of stochastic approximations where one system can be seen to update more aggressively than the other then the aggressive process is always fully adjusted to the value of the other process. This is all controlled through the user's choice of step sizes in the stochastic approximation. The main result in this section is Corollary \ref{twotimescaleconvergence}, which comes from combining Theorem \ref{singletimescaletheorem} with the previous work of Konda and Borkar \cite{KondaBorkar}.

\subsection{Notation}

In what is to follow we consider the extension of Theorem \ref{singletimescaletheorem} to the two-timescales setting, with updates $\{x_n\}_{n \in \mathbb{N}}$ and $\{y_n\}_{n \in \mathbb{N}}$ on different timescales. Let $I$ be the set of individual elements of the $x$ process as in Section \ref{sec:AsynchSA}, and define $J$ similarly for the $y$ process. Let $K = |I|$ and $L = |J|$ so that for all $n$, $x_n \in \mathbb{R}^K$ and $y_n \in \mathbb{R}^L$. As in Section \ref{sec:AsynchSA} let $\bar{I} \subset 2^I$ be the set containing all combinations of elements in $I$ which have a positive probability of being part of the asynchronous update, and define $\bar{J} \subset 2^J$ in the same manner for the $y$ process. At iteration $n$ let $\bar{I}_n \in \bar{I}$ and $\bar{J}_n \in \bar{J}$ be the updated components of each timescale respectively. Let each component of the two processes have a counter for the number of times it has been selected to be updated defined by,

\begin{equation*}
	\nu_n(i) := \sum_{k=1}^n \mathbb{I}_{\{i \in \bar{I}_k\}}, \quad \phi_n(j) := \sum_{k=1}^n \mathbb{I}_{\{j \in \bar{J}_k\}}.
\end{equation*}
Here $\nu_n(i)$ is as in Section \ref{sec:AsynchSA} and $\phi_n(j)$ has an analogous definition for the $\{y_n\}_{n \in \mathbb{N}}$ process. Let $\{V_n\}_{n \in \mathbb{N}}$, $\{U_n\}_{n \in \mathbb{N}}$ be martingale noise processes defined on $\mathbb{R}^K$ and $\mathbb{R}^L$ respectively, and $\{d_{n}\}_{n \in \mathbb{N}}$, $\{e_n\}_{n \in \mathbb{N}} \rightarrow 0$ as $n \rightarrow \infty$ similarly defined on $\mathbb{R}^K$ and $\mathbb{R}^L$ respectively. Let $V_n(i), d_n(i) \in \mathbb{R}$ be component $i$ of $V_n$ and $d_n$, and similarly let $U_n(j), e_n(j) \in \mathbb{R}$ be component $j$ of $U_n$ and $e_n$. As in the previous sections $\{\alpha(n)\}_{n \in \mathbb{N}}$, and now $\{\gamma(n)\}_{n \in \mathbb{N}}$, are positive, deceasing sequences of step sizes. Similar restrictions to those in (A2) will be placed on $\{\alpha(n)\}_{n \in \mathbb{N}}$, $\{\gamma(n)\}_{n \in \mathbb{N}}$ with an additional requirement for the two-timescale arguments to be valid; this will be made precise in Section \ref{sec:2timeAssump}. Finally, $F(\cdot,\cdot): \mathbb{R}^K \times \mathbb{R}^L \rightarrow \mathbb{R}^K$ and $G(\cdot,\cdot): \mathbb{R}^K \times \mathbb{R}^L \rightarrow \mathbb{R}^L$ are set-valued maps, where $F_i(x,y)$ is the $i^{\mathrm{th}}$ value of $F(x,y)$ and similarly for $G_j(x,y)$. For all $i = 1 \ldots, K$ and $j = 1,\ldots, L$ consider the following coupled process,

\begin{equation}
	{\setlength\arraycolsep{0.1em}
	\begin{array}{rl}
		x_{n+1}(i) & - x_n(i) \\
		& \in \alpha \big(\nu_{n+1}(i)\big) \mathbb{I}_{\{i \in \bar{I}_{n+1}\}} \big[F_i(x_n,y_n) + V_{n+1}(i) + d_{n+1}(i) \big], \\
		y_{n+1}(j) & - y_n(j) \\
		& \in \gamma \big(\phi_{n+1}(j)\big) \mathbb{I}_{\{j \in \bar{J}_{n+1}\}} \big[G_j(x_n,y_n) + U_{n+1}(j) + e_{n+1}(j) \big].
	\end{array} 
	}
	\label{2asynchronousSAlong}
\end{equation}
Notice that the only change to the first process from Sections \ref{sec:Notation} and \ref{sec:AsynchSA} is that the mean field now depends on $y_n$ as well as $x_n$. It follows that the asynchronous and relative step sizes retain the same form. Recall these definitions and extend them for the $\{y_n\}_{n \in \mathbb{N}}$ process:

\begin{align*}
	\bar{\alpha}_n := \max_{i \in \bar{I}_{n}} \alpha \big(\nu_{n}(i)\big), & \quad \mu_n(i) := \frac{\alpha\big(\nu_{n}(i)\big) }{\bar{\alpha}_n} \mathbb{I}_{\{i \in \bar{I}_{n}\}}, \\
	\bar{\gamma}_n := \max_{j \in \bar{J}_{n}} \gamma \big(\phi_{n}(j)\big), & \quad \sigma_n(j) := \frac{\gamma\big(\phi_{n}(j)\big) }{\bar{\gamma}_n} \mathbb{I}_{\{j \in \bar{J}_{n}\}}.
\end{align*}
As in Section \ref{sec:BackgroundAsynch} let $M_n$ be the $K \times K$ diagonal matrix of the $\mu_n(i)$ terms and similarly let $N_n$ be the $L \times L$ diagonal matrix of the $\sigma_n(j)$ terms. The coupled stochastic process \eqref{2asynchronousSAlong} can be written more concisely as,

\begin{equation}
	{\setlength\arraycolsep{0.1em}
	\begin{array}{rl}
		x_{n+1} - x_n - \bar{\alpha}_{n+1} M_{n+1} \big[ V_{n+1} + d_{n+1} \big] & \in \bar{\alpha}_{n+1} M_{n+1} \cdot F(x_n,y_n), \\
		y_{n+1} - y_n - \bar{\gamma}_{n+1} N_{n+1} \big[ U_{n+1} + e_{n+1} \big] & \in \bar{\gamma}_{n+1} N_{n+1} \cdot G(x_n,y_n).
	\end{array} 
	}
	\label{2asynchronousSA}
\end{equation}
Finally, define the two timescales; let $\bar{\tau}_0 := 0$, $\bar{\tau}_k := \sum_{i=1}^{k} \bar{\alpha}_i$, $\bar{\rho}_0 := 0$, $\bar{\rho}_k := \sum_{i=1}^{k} \bar{\gamma}_i$. The division of time on the `slow' timescale is given by the increments $\{\bar{\tau}_n\}_{n \in \mathbb{N}}$ and similarly for $\{\bar{\rho}_n\}_{n \in \mathbb{N}}$ on the `fast' timescale. In a similar manner to the previous sections let $\bar{m}_{\alpha}(t) := \sup \{k \geq 0 ; t \geq \bar{\tau}_k\}$, and $\bar{m}_{\gamma}(t) := \sup \{k \geq 0 ; t \geq \bar{\rho}_k\}$.

%%%%%%%%%%%%%%%%%%%%%%%%%%%%
%%%%%%%%%%%%%%%%%%%%%%%%%%%%
%%%%%%%%%%%%%%%%%%%%%%%%%%%%
%%%%%%%%%%%%%%%%%%%%%%%%%%%%
%%%%%%%%%%%%%%%%%%%%%%%%%%%%
%%%%%%%%%%%%%%%%%%%%%%%%%%%%

\subsection{Assumptions}
\label{sec:2timeAssump}

We state the assumptions (B1)-(B6) used for the convergence results of the two-timescale algorithm \eqref{2asynchronousSA}. These are exactly analogous to (A1)-(A5) and are simply extended to accommodate the two-timescales framework. The exceptions to this are (B2)(c) and (B6) and the slight adaptations to (B3), which are in line with those used by Borkar \cite{BorkarTimescales} and Konda and Borkar \cite{KondaBorkar}. In (B4) we have produced a single combined Markov chain instead of one for each of the $\{x_n\}_{n \in \mathbb{N}}$ and $\{y_n\}_{n \in \mathbb{N}}$ processes to present a clearer assumption.

\begin{itemize}
	\item[(B1)] 
	\begin{itemize}
		\item[(a)] For compact sets, $C \subset \mathbb{R}^K$, $D \subset \mathbb{R}^L$, $x_n \in C$, $y_n \in D$ for all $n$.
		\item[(b)] $\{d_n\}_{n \in \mathbb{N}}$ and $\{e_n\}_{n \in \mathbb{N}}$ are bounded sequences such that $d_n,e_n \rightarrow 0$ as $n \rightarrow \infty$.
	\end{itemize}
	\item[(B2)] The following must be true for $a(n) = \alpha(n)$ and $a(n) = \gamma(n)$
	\begin{itemize}
		\item[(a)] $ \sum_n a(n) = \infty$ and $a(n) \rightarrow 0$ as $n \rightarrow \infty$,
		\item[(b)] For $x \in (0,1)$, $\sup_n a([xn]) / a(n) < A_x < \infty$. In addition, for all $n$, $a(n) \geq a(n+1)$. $[\cdots]$.
		\item[(c)] $\frac{\alpha(n)}{\gamma(n)} \rightarrow 0$ 
	\end{itemize}
	\item[(B3)]
		\begin{itemize}
			\item[(a)] For all $z \in \{(x,y) ; x \in C, y \in D\}$, $F(z)$ is upper semi-continuous and for all $y \in D$, $F(\cdot,y): \mathbb{R}^K \rightarrow \mathbb{R}^K$ is a stochastic approximation map. 
			\item[(b)] For all $z \in \{(x,y) ; x \in C, y \in D\}$, $G(z)$ is a stochastic approximation map
		\end{itemize}
\end{itemize}

The first and second assumptions are direct extensions of (A1) and (A2) to two-timescales with the addition of (B2)(c) which is a standard two-timescale assumption used by Borkar \cite{BorkarTimescales}. Condition (B3)(a) is similar to (A3) for the `slow' timescale, however this must hold for all values of the `fast' timescale. (B3)(b) is a similar condition for the `fast' timescale.

Define $\bar{H} \subset \bar{I} \times \bar{J}$ such that if $\mathcal{I} \in \bar{I}$ and $\mathcal{J} \in \bar{J}$ then $(\mathcal{I},\mathcal{J}) \in \bar{H}$ if and only if $\mathcal{I}$ and $\mathcal{J}$ have a positive probability of occurring simultaneously (at the same iteration). This means that $\bar{H}$ is the combination of elements across $\bar{I} \times \bar{J}$ which have positive probability of being updated at any particular iteration. At iteration $n$ $\bar{H}_n \in \bar{H}$ is taken to be the updated components across $I$ and $J$. In addition, let $z_n = (x_n,y_n) \in C \times D$ and $\mathcal{F}_{n}$ be a sigma algebra containing all the information up to and including the $n^{\mathrm{th}}$ iteration. That is $\mathcal{F}_{n} = \sigma(\{\bar{H}_m\}_m,\{z_m\}_m,\{\nu_m(i)\}_{i,m},\{\phi_m(j)\}_{j,m}; \forall m \leq n, i = 1,\ldots,K, j = 1,\ldots,L)$. 

\begin{itemize}
	\item[(B4)]
	\begin{itemize}
		\item[(a)] For all $z \in C \times D$ and $\mathcal{H}_n, \mathcal{H}_{n+1} \in \bar{H}$,
		\begin{equation*}
			\mathbb{P}\Big(\bar{H}_{n+1}=\mathcal{H}_{n+1} |\mathcal{F}_{n} \Big) = \mathbb{P}\Big(\bar{H}_{n+1}=\mathcal{H}_{n+1} |\bar{H}_n=\mathcal{H}_n,z_n=z \Big).
		\end{equation*}
	\end{itemize}
	Let
	\begin{equation*}
		Q_{(\mathcal{H}_n, \mathcal{H}_{n+1})}(z) := \mathbb{P}\Big(\bar{H}_{n+1}=\mathcal{H}_{n+1} |\bar{H}_n=\mathcal{H}_n,z_n=z \Big).
	\end{equation*}
	\begin{itemize}
		\item[(b)] For all $z \in C \times D$ the transition probabilities $Q_{(\mathcal{H}_n, \mathcal{H}_{n+1})}(z)$ form aperiodic, irreducible, positive recurrent Markov chains over $\bar{H}$ and for all $i \in I$, $j \in J$ there exists a $\mathcal{H},\mathcal{H}' \in \bar{H}$, such that $i \in \mathcal{H}$, $j \in \mathcal{H}'$.
		\item[(c)] the map $z \mapsto Q_{(\mathcal{H}_n, \mathcal{H}_{n+1})}(z)$, is Lipschitz continuous.
	\end{itemize}
\end{itemize}

\begin{itemize}
	\item[(B5)] 
		Separately for both 
		
		\begin{equation*}
			(a(n),W_n, \bar{Z}_n) = (\alpha(n),V_n,\bar{I}_n) \quad \mbox{and} \quad (a(n),W_n, \bar{Z}_n) = (\gamma(n),U_n,\bar{J}_n) 
		\end{equation*}
		one of the following assumptions is satisfied:

	\begin{itemize}
		\item[(a)] For some $q \geq 2$
			
			\begin{equation*}
				\sum_n a(n)^{1 + q/2} < \infty, \quad \mbox{and} \quad \sup_n \mathbb{E}\big(\|W_n \|^q\big) < \infty;
			\end{equation*}
		\item[(b)] With $W_n(i)$ independent of $\mathbb{I}_{\{i \in \bar{Z}_n\}}$ given $\mathcal{F}_{n-1}$, $W_n(i)$ independent of $W_n(j)$ for $i \neq j$ and $\langle a,b \rangle = \sum_k a_k b_k$. There exists a positive $\Gamma$ such that for all $\theta \in \mathbb{R}^K / \mathbb{R}^L$ (depending on the dimension of $W_n$),
		
			\begin{equation*}
				\sum_n e^{-c/a(n)} < \infty,
			\end{equation*}
			and
			\begin{equation*}
				\mathbb{E}\Big[\exp\Big\{\langle \theta,W_{n+1} \rangle \Big\} | \mathcal{F}_n\Big] \leq \exp\Big\{\frac{\Gamma}{2} \|\theta\|^2 \Big\};
			\end{equation*}
			for each $c > 0$.
	\end{itemize}
	If (B5)(a) holds for $(\alpha(n), V_n, \bar{I}_n)$ and (B5)(b) holds for $(\gamma(n), U_n, \bar{J}_n)$ (or vice versa) then (B5) holds.
\end{itemize}	

(B4) and (B5) are straightforward extensions to (A4) and (A5) where in (B4) we have chosen to create a combined Markov chain over both $\bar{I}$ and $\bar{J}$ to present a clearer assumption. As a result of (B4), Lemma \ref{weakconvergence} gives that every element of $\bar{H}$ (and hence every element of $I$ and $J$) is updated some minimum proportion of the time in the limit (see Section \ref{section:AsynchronousUpdates}). Let $\varepsilon > 0$ be this minimum proportion. Define $\bar{F}(x,y) := \Omega^{\varepsilon}_K \cdot F(x,y)$ and $\bar{G}(x,y) := \Omega^{\varepsilon}_L \cdot G(x,y)$ analogously to the definition of $\bar{F}(\cdot)$ in \eqref{Fbar}, where $\Omega^{\varepsilon}_K$, $\Omega^{\varepsilon}_L$ are as defined in \eqref{OmegaSetdfn}.

\begin{itemize}
	\item[(B6)] For all $x \in C$ the differential inclusion,
	
		\begin{equation*}
			\frac{d y}{d t} \in \bar{G}(x,y), 
		\end{equation*}
		has a unique globally asymptotic stable equilibrium, $\Lambda(x)$, where $\Lambda(\cdot) : \mathbb{R}^K \rightarrow \mathbb{R}^L$ is bounded, continuous and single-valued for $x \in C$.
\end{itemize}

The final assumption, (B6), is the asynchronous equivalent to the `fast' timescale convergence criteria used by Borkar \cite{BorkarTimescales} and we use throughout this section. However, at the end of Section \ref{sec:slowtimescaleconvergence} we provide an alternative assumption, (B6$'$), which allows the `fast' timescale to converge instead to a globally attracting set, as opposed to a continuous single-valued function.

%%%%%%%%%%%%%%%%%%%%%%%%%%%%
%%%%%%%%%%%%%%%%%%%%%%%%%%%%
%%%%%%%%%%%%%%%%%%%%%%%%%%%%
%%%%%%%%%%%%%%%%%%%%%%%%%%%%
%%%%%%%%%%%%%%%%%%%%%%%%%%%%
%%%%%%%%%%%%%%%%%%%%%%%%%%%%

\subsection{Convergence of the `Fast' Timescale}
\label{sec:fasttimescaleconvergence}

Many of the proofs in this section use arguments which are identical to the corresponding results of Section \ref{sec:AsynchSA}; where this is the case we do not go into detail and instead direct the reader to the appropriate result(s) to identify the method used. 

Firstly, we require an additional lemma used by Konda and Borkar \cite{KondaBorkar} which shows that the key two-timescales arguments made by Borkar \cite{BorkarTimescales} will continue to hold in the asynchronous case.

\begin{twotimescaleslemma}
	\label{randomstepsize}
	Under (B2) and (B4) $\bar{\alpha}_n, \bar{\gamma}_n \rightarrow 0$ and $\frac{\bar{\alpha}_n}{\bar{\gamma}_n} \rightarrow 0$ almost surely.
\end{twotimescaleslemma}

\begin{proof}
	The proof of this result is identical to \cite[lemma 4.6]{KondaBorkar}. The requisite assumptions are encapsulated in (B2) and the result of Lemma \ref{minupdateproportion}, in Appendix \ref{MarkovUpdates}.
\end{proof}

Again, we follow the method of Borkar \cite{BorkarBook} when examining the set-valued updates. Define $f_n$ and $g_n$ in the following manner; for $i = 1,\ldots,K$, $j = 1,\ldots,L$ let

\begin{align}
	f_n(i) \mu_{n+1}(i) & := \frac{x_{n+1}(i) - x_n(i)}{\bar{\alpha}_{n+1}} - \mu_{n+1}(i) \big[ V_{n+1}(i) + d_{n+1}(i) \big], \notag \\ 
	g_n(j) \sigma_{n+1}(j) & := \frac{y_{n+1}(j) - y_n(j)}{\bar{\gamma}_{n+1}} - \sigma_{n+1}(j) \big[ U_{n+1}(j) + e_{n+1}(j)\big]. \notag
\end{align}
As in Section \ref{sec:MainResult}, if $\mu_{n+1}(i) = 0$ then we can select any $f_n(i) \in F_i(x_n,y_n)$ and similarly if $\sigma_{n+1}(j) = 0$ then we can select any $g_n(j) \in G_j(x_n,y_n)$. $f_n$ and $g_n$ represent the realised values of $F(x_n,y_n)$ and $G(x_n,y_n)$ respectively. We express \eqref{2asynchronousSA} as,

\begin{equation}
	{\setlength\arraycolsep{0.1em}
	\begin{array}{rl}
		x_{n+1} & = x_n + \bar{\alpha}_{n+1} M_{n+1} \big[f_n + V_{n+1} + d_{n+1} \big], \\
		y_{n+1} & = y_n + \bar{\gamma}_{n+1} N_{n+1} \big[g_n + U_{n+1} + e_{n+1} \big].
	\end{array} 
	}
	\label{2asynchronousSA2}
\end{equation}
Let $\underline{0}^K$ be a zero vector of length $K$, and define

\begin{align*}
	& \hspace{1mm} 
	z_n := 	
	\left(
		\begin{array}{c}
			x_n \\
			y_n
		\end{array} 
	\right),
	\quad
	\Gamma_n := 	
	\left(
		\begin{array}{c}
			M_{n} \\
			N_{n}
		\end{array} 
	\right),
	\quad
	\zeta_{n} := 	
	\left(
		\begin{array}{c}
			\frac{\bar{\alpha}_{n}}{\bar{\gamma}_{n}} V_n \\
			U_n
		\end{array} 
	\right), \\
	& 
	\epsilon_{n+1} := 	
	\left(
		\begin{array}{c}
			\frac{\bar{\alpha}_{n}}{\bar{\gamma}_{n}} (f_n + d_{n+1}) \\
			e_{n+1}
		\end{array} 
	\right),
	\quad\quad
	\Psi(z_n) :=
	\left(
		\begin{array}{c}
			\underline{0}^K \\
			G(z_n)
		\end{array} 
	\right).	
\end{align*}
Then we can express the coupled process in \eqref{2asynchronousSA2} as a single iterative process,

\begin{equation}
	z_{n+1} - z_n - \bar{\gamma}_{n+1} \Gamma_{n+1} \big[ \zeta_{n+1} + \epsilon_{n+1}\big] \in \bar{\gamma}_{n+1} \Gamma_{n+1} \cdot \Psi(z_n). \label{2timescalesfast}
\end{equation}
This is in the same form as equation \eqref{asynchronousSAmatrix} which is examined throughout Section \ref{sec:AsynchSA}. The limiting behaviour of \eqref{2timescalesfast} can therefore be studied using the same method as in Section \ref{sec:AsynchSA}. Although assumptions (A1)-(A4) are embedded in (B1)-(B6), we negate the need for (A5) under (B1)-(B6) by immediately proving the corresponding result in Lemma \ref{lemma:Noise} for \eqref{2timescalesfast}.

\begin{Lem:TimescalesError}
	\label{Lem:TimescalesError}
	Under (B2)(b),(B4) and (B5), with probability 1,

	\begin{equation*}
		\lim_{n \rightarrow \infty} \sup_{k} \left\{ \left\| \sum_{j=n}^{k-1} \bar{\gamma}_{j+1} \Gamma_{j+1} \zeta_{j+1}\right\| ; k = n+1, \ldots, \bar{m}_{\gamma}(\bar{\rho}_n + T)\right\} = 0
	\end{equation*}
\end{Lem:TimescalesError}

\begin{proof}
Firstly, let $\bar{\xi}^{\alpha}_n(T) := \sup\{k \geq 0; t \geq \bar{\tau}_{n+k}\}$ and $\bar{\xi}^{\gamma}_n(T) := \sup\{k \geq 0; t \geq \bar{\rho}_{n+k}\}$. By Lemma \ref{randomstepsize}, $\bar{\gamma}_n > \bar{\alpha}$ in the limit and hence $\bar{\xi}^{\alpha}_n(t) \geq \bar{\xi}^{\gamma}_n(t)$. From this, for limiting values of $n$,

\begin{equation}
	\bar{m}_{\alpha}(\bar{\tau}_n + T) = n + \bar{\xi}^{\alpha}_n(T) \geq n + \bar{\xi}^{\gamma}_n(T) = \bar{m}_{\gamma}(\bar{\rho}_n + T). \label{timescalesstepsize}
\end{equation}
Now,

\begin{align*}
	\left\| \sum_{j=n}^{k-1} \bar{\gamma}_{j+1} \Gamma_{j+1} \zeta_{j+1}\right\| & \leq \left\| \sum_{j=n}^{k-1} \bar{\alpha}_{j+1} M_{j+1} V_{j+1}\right\| \\
	& \quad + \left\| \sum_{j=n}^{k-1} \bar{\gamma}_{j+1} N_{j+1} U_{j+1}\right\|.
\end{align*}
Using that the assumptions for Lemma \ref{lemma:Noise} are contained in (B2)(b), (B4) and (B5), the first term converges to zero for $k = n+1, \ldots, \bar{m}_{\alpha}(\bar{\tau}_n + T)$ and with the same arguments, the second term converges to zero for $k = n+1, \ldots, \bar{m}_{\gamma}(\bar{\rho}_n + T)$. Combining this with \eqref{timescalesstepsize} proves the result for $k = n+1, \ldots, \bar{m}_{\gamma}(\bar{\rho}_n + T)$ as required.
\end{proof}

Now we note that assumption (B4) allows us to use identical arguments to those in Lemma \ref{weakconvergence}. Recall the definition of $\Omega^{\delta}_k$ in \eqref{OmegaSetdfn}. Then there exists an $\varepsilon > 0$ such that $\Gamma_n \in \Omega^{\varepsilon}_{K+L}$. Let $\bar{z}(t)$ be the linear interpolation of \eqref{2timescalesfast} in the same manner as the single timescale case in \eqref{interpolatedsA}. Fix $\varepsilon > 0$ and let $\bar{\Psi}(\cdot) := \Omega^{\varepsilon}_{K+L} \cdot \Psi(\cdot)$, $\bar{\Psi}(\cdot): \mathbb{R}^{K+L} \rightarrow \mathbb{R}^{K+L}$. 

\begin{twotimescalestheorem}
	\label{twotimescalestheorem}
		Under assumptions (B1)-(B6), with probability 1, $\bar{z}(t) \in \mathbb{R}^{K+L}$ is an asymptotic pseudo-trajectory of the differential inclusion,
	
	\begin{equation}
		\frac{d z}{\d t} \in \bar{\Psi}(z).
	\end{equation}
\end{twotimescalestheorem}

\begin{proof}
	We have shown (A1)-(A4) hold for \eqref{2timescalesfast} and this is in the form of \eqref{asynchronousSAmatrix}. Combining this and Lemma \ref{Lem:TimescalesError} we can use Theorem \ref{singletimescaletheorem} to give the result immediately.
\end{proof}

Let the linear interpolations of the two-timescales in \eqref{2asynchronousSA2} be denoted by $\bar{x}(t)$ and $\bar{y}(t)$ respectively analogously to the single timescale case in \eqref{interpolatedsA}.

\begin{twotimescalescorollary}
	\label{twotimescalescorollary}
		Under assumptions (B1)-(B6), with probability 1, the interpolated process
		
		\begin{equation*}
		\big(\bar{x}(t),\bar{y}(t)\big) \rightarrow \Big\{\big(x,\Lambda(x)\big) ; x \in C \Big\} \mbox{ as } t \rightarrow \infty.
		\end{equation*}
\end{twotimescalescorollary}

\begin{proof}
	Immediate from Lemma \ref{twotimescalestheorem} and (B6) using the same arguments as Borkar \cite{BorkarTimescales}.
\end{proof}

%%%%%%%%%%%%%%%%%%%%%%%%%%%%
%%%%%%%%%%%%%%%%%%%%%%%%%%%%
%%%%%%%%%%%%%%%%%%%%%%%%%%%%
%%%%%%%%%%%%%%%%%%%%%%%%%%%%
%%%%%%%%%%%%%%%%%%%%%%%%%%%%
%%%%%%%%%%%%%%%%%%%%%%%%%%%%

\subsection{Convergence of the `Slow' Timescale}
\label{sec:slowtimescaleconvergence}

Now since we have a function $F(\cdot,\cdot)$ which depends on two variables, but we are treating one of these as fully calibrated to the other, we have a slightly different framework to that of Benaim, Hofbauer and Sorin \cite{BenaimHofbauerSorin}. Therefore we present a slight variation on their perturbed solution \cite[Definition (II)]{BenaimHofbauerSorin}. Despite this we are still able to show that this is an asymptotic pseudo-trajectory to the desired differential inclusion. Hence the same convergence results still apply. To reduce notation, define $F^{\Lambda}(\cdot): \mathbb{R}^K \rightarrow \mathbb{R}^K$ as $F^{\Lambda}(x) := F\big(x,\Lambda(x)\big)$. Note that under (B3)(a) and (B6) $F^{\Lambda}(\cdot)$ is a stochastic approximation map. Following our previous notation, let 

\begin{equation*}
	\bar{F}^{\Lambda}(x) := \Omega^{\varepsilon}_K \cdot F^{\Lambda}(x).
\end{equation*}

\begin{JointlyPerturbedDefn}
	\label{JointlyPerturbedDefn}
	A continuous function $z:\mathbb{R}^+ \rightarrow \mathbb{R}^K$ is a \emph{jointly perturbed solution} to the differential inclusion,
	
	\begin{equation}
		\frac{d x}{dt} \in \bar{F}^{\Lambda}(x) \label{FbarLambdaInclusion},
	\end{equation} 
	if 
	
	\begin{itemize}
		\item[(i)] $z$ is absolutely continuous.
		\item[(ii)] $\frac{d z(t)}{dt} - V(t) - d(t) \in \bar{F}^{\Lambda}_{\delta(t)}\big(z(t)\big)$, for almost every $t > 0$, and for some bounded $d(t), \delta(t) \rightarrow 0$, where
		
		\begin{align*}
			\bar{F}^{\Lambda}_{\delta}(x) = \{f \in C \hspace{1mm} & ; \|x'-x\|\leq\delta, y' \in D, \\
			& \|y' - \Lambda(x')\| \leq \delta, \inf_{f' \in \bar{F}(x',y')} \|f'-f\| \leq \delta \}.
		\end{align*}
		\item[(iii)] $t \rightarrow V(t)$ is a locally integrable function such that for all $T > 0$
		
		\begin{equation*}
			\lim_{t \rightarrow \infty} \sup_{0 \leq v \leq T} \left\| \int_t^{t+v} V(s) \d s \right\| = 0.
		\end{equation*}
	\end{itemize}
\end{JointlyPerturbedDefn}

The key difference between a jointly perturbed solution and a perturbed solution is that in the original work of Benaim, Hofbauer and Sorin \cite{BenaimHofbauerSorin} (and as in Section \ref{sec:AsynchSA}) the mean field depends on a single variable. In contrast to this, here the mean field depends on two variables. Hence in part (ii) we must allow for perturbations in both variables simultaneously instead of perturbing just the one.

\begin{JointlyPerturbedAPTEquiv}
	\label{JointlyPerturbedAPTEquiv}
	Under assumptions (B1)-(B6) a jointly perturbed solution of \eqref{FbarLambdaInclusion} is also an asymptotic pseudo-trajectory to the flow induced by \eqref{FbarLambdaInclusion}.
\end{JointlyPerturbedAPTEquiv}

\begin{proof}
	The proof is identical to the proof of \cite[Theorem 4.2]{BenaimHofbauerSorin} which establishes that a perturbed solution is an asymptotic pseudo-trajectory.
\end{proof}

Define $\tilde{M}(t)$ and $\tilde{M}_n$ in an identical manner to the same terms in Section \ref{sec:MainResultProof}. Corollary \ref{twotimescalescorollary} allows us to consider the updates on the `slow' timescale of the coupled process in \eqref{2asynchronousSA2} given by,

\begin{equation}
	x_{n+1} - x_n - \bar{\alpha}_{n+1} \left[ \bar{V}_{n+1} + \bar{d}_{n+1} \right] \in \bar{\alpha}_{n+1} \bar{F}(x_n,y_n), \label{2timescalesslow}
\end{equation} 
where, as in Section \ref{sec:MainResultProof}, $\bar{V}_{n+1} = M_{n+1} V_{n+1} + [M_{n+1} - \tilde{M}_{n+1}] f_n$ and $\bar{d}_{n+1} = M_{n+1} d_{n+1}$. As with the single timescale framework let $\bar{x}(t)$ be a linear interpolation of \eqref{2timescalesslow} in the same manner as \eqref{interpolatedsA}. That is,

\begin{equation*}
	\bar{x}(\bar{\tau}_n+s) = x_n + s \frac{x_{n+1} - x_n}{\bar{\alpha}_{n+1}}, \quad s \in \left[0,\bar{\alpha}_{n+1}\right). 
\end{equation*}

\begin{xbarJointlyPerturbedSol}
	\label{xbarJointlyPerturbedSol}
	Under assumptions (B1)-(B6), with probability 1, $\bar{x}(\cdot)$ is an asymptotic pseudo-trajectory to the differential inclusion \eqref{FbarLambdaInclusion}.
\end{xbarJointlyPerturbedSol}

\begin{proof}
	We show that $\bar{x}(\cdot)$ is a jointly perturbed solution of \eqref{FbarLambdaInclusion} and hence by Lemma \ref{JointlyPerturbedAPTEquiv} it is an asymptotic pseudo-trajectory of \eqref{FbarLambdaInclusion}. 
	
	The proof is almost identical to the proof of \cite[Proposition 1.3]{BenaimHofbauerSorin}. The differences come with our choice of $\delta(t) = \|\bar{x}(t)-x_{m(t)}\|+ \|y_{m(t)} - \Lambda(x_{m(t)})\|$, the first term of which will converge to zero as in the original proof and the second of which converges to zero as $\|y_n - \Lambda(x_n)\| \rightarrow 0$ almost surely as a result of the convergence of the `fast' timescale. Clearly from Definition \ref{JointlyPerturbedDefn} part (ii), $F(x_{m(t)},y_{m(t)}) \subset F^{\Lambda}_{\delta(t)}(\bar{x}(t))$. The rest of the proof completes as in \cite[Proposition 1.3]{BenaimHofbauerSorin}.
\end{proof}

\begin{twotimescaleconvergence}
	\label{twotimescaleconvergence}
	If there is a globally attracting set, $A$, for the differential inclusion \eqref{FbarLambdaInclusion} and assumptions (B1)-(B6) are satisfied, then the two-timescale iterative process \eqref{2asynchronousSA} will almost surely converge to $A$.
\end{twotimescaleconvergence}
	
\begin{proof}
	Immediate by combining Corollary \ref{twotimescalescorollary} and Theorem \ref{xbarJointlyPerturbedSol} with \cite[Proposition 3.27]{BenaimHofbauerSorin}.
\end{proof}

\begin{nonasynchronous}
	It should be clear that the methods in this chapter can be applied to a standard, synchronous, stochastic approximation with set-valued mean fields. In this case (B4) is trivially satisfied, (B2)(b) can be removed and the use of the sets $\Omega^{\varepsilon}_K$, $\Omega^{\varepsilon}_L$ can be replaced by a single $K \times K$ and $L \times L$ identity matrix respectively. This means that $\bar{F}(x,y) = F(x,y)$ and similarly $\bar{G}(x,y) = G(x,y)$.
\end{nonasynchronous}

\begin{setvaluedlimit}
	This framework allows for the `fast' timescale to converge to a set of limit points. Assumption (B6) can be replaced with the following
	
	\begin{itemize}
		\item[(B6$'$)] For all $x \in C$ the differential inclusion,
	
		\begin{equation*}
			\frac{d y}{d t} \in \bar{G}(x,y) 
		\end{equation*}
		has a globally attracting set, $\Lambda(x)$ where	$\Lambda(\cdot) : \mathbb{R}^K \rightarrow \mathbb{R}^L$ is an upper semi-continuous set valued map, such that for all $x \in C$, $\Lambda(x)$ is compact, convex and non-empty. In addition, for all $x \in C$ and for all, $F(x,\Lambda(x))$ is a convex set-valued map, i.e., for all $\lambda, \lambda' \in \Lambda(x)$,
		
		\begin{equation*}
			\alpha F(x, \lambda) + (1 - \alpha) F(x, \lambda') \subset  F\big(x, \alpha \lambda + (1-\alpha) \lambda'\big), \quad \mbox{ for all } \alpha \in [0,1].
		\end{equation*}
	\end{itemize}
	
	Under (B6$'$) $F^{\Lambda}(x)$ will still be a stochastic approximation map and, although the details of the arguments in Section \ref{sec:slowtimescaleconvergence} change slightly, the method remains almost identical.
\end{setvaluedlimit}

%%%%%%%%%%%%%%%%%%%%%%%%%%%%
%%%%%%%%%%%%%%%%%%%%%%%%%%%%
%%%%%%%%%%%%%%%%%%%%%%%%%%%%
%%%%%%%%%%%%%%%%%%%%%%%%%%%%
%%%%%%%%%%%%%%%%%%%%%%%%%%%%
%%%%%%%%%%%%%%%%%%%%%%%%%%%%

\section{An Application: Learning in a Markov Decision Process}
\label{sec:MDPApplication}

In this section we provide an example where our two-timescale asynchronous stochastic approximation approach is needed. The algorithm is an actor-critic style algorithm based upon estimating rewards and playing an $\varepsilon$-greedy best response to these estimates. Similarly to Konda and Borkar \cite{KondaBorkar} we make assumptions on the Markov decision process (MDP) and the coupled learning algorithm separately. These correspond to (B1)-(B6).

Firstly, we begin by outlining a suitable infinite horizon, discounted reward MDP described by the tuple $\left\langle \mathcal{S},\mathcal{A},P,r,\beta\right\rangle$. $\mathcal{S}$ is the state space of the MDP, $\mathcal{A}$ is the set of actions of the decision-maker (agent), and $\mathcal{A}(s)$ is the set of actions available to the agent in state $s \in \mathcal{S}$. $P$ represents the form of the stochastic transitions and in what is to follow we take $P_{ss'}(a)$ to denote probability of transitioning from state $s \in \mathcal{S}$ to $s' \in \mathcal{S}$ when the agent has selected action $a \in \mathcal{A}(s)$. The reward to the decision maker for selecting action $a \in \mathcal{A}(s)$ is denoted by $r(s,a)$, and $\beta \in (0,1)$ is the discount factor. 

Let $s_n \in \mathcal{S}$ and $a_n \in \mathcal{A}(s_n)$ be the state and the action, respectively, selected by the decision-maker at iteration $n$. Assume that at every iteration the agent observes the state of the process and a noisy version of the reward received from the action they have chosen, denoted by $R_n$. If $\mathcal{F}_n := \{a_m,s_m,R_{m-1} ; m=1,\ldots,n\}$ then $\mathbb{E}[R_n|\mathcal{F}_n] =  r(s_n,a_n)$, and we assume that $R_n$ has a finite variance. Let $K := |\mathcal{S}|$ and $\Delta\big(A(s)\big)$ represent the set of probability distributions over $\mathcal{A}(s)$. Then we denote the combination of $K$ probability distributions as $\Delta_K := \Delta\big(\mathcal{A}(1)\big)\times \ldots \times \Delta\big(\mathcal{A}(K)\big)$. A strategy for state $s \in \mathcal{S}$ is denoted by $\pi(s) \in \Delta\big(\mathcal{A}(s)\big)$ and let $\pi := \big(\pi(1),\ldots,\pi(K)\big) \in \Delta_K$ be a strategy over all states. $\pi(s,a)$ is defined as the probability that action $a$ is taken in state $s$. Players start with a strategy $\pi_0$; the MDP begins in a random state $s_1 \in \mathcal{S}$ and the decision-maker selects an action $a_1$ from $\pi_0$. The agent wishes to find a strategy, $\pi$, to maximise their expected discounted reward,

\begin{equation*}
	\mathbb{E}\Big[\sum_{n=1}^{\infty} \beta^n r(s_n,a_n)\Big].
\end{equation*}
Define $V^{\pi}(s)$ for all $s \in \mathcal{S}$ in the following manner,

\begin{equation}
	V^{\pi}(s) = \sum_{a \in \mathcal{A}(s)} \pi(s,a) \left[ r(s,a) + \omega \sum_{s' \in S} P_{ss'}(a) V^{\pi}(s) \right]. \label{valuefunctiondfn}
\end{equation}
For all $s \in S$ there exists a maximum value of $V^{\pi}(s)$ for all $\pi \in \Delta_K$ \cite{Bellman}. Let $\tilde{\pi} \in \Delta_K$ be a strategy such that $V^{\tilde{\pi}}(s)$ is maximal in \eqref{valuefunctiondfn} for all $s \in \mathcal{S}$. $\tilde{\pi}$ is known as an \emph{optimal strategy} for the MDP. Again we use a subscript $n$ on $\pi(s,a)$, $\pi(s)$ and $\pi$ to represent the strategy of the player at iteration $n$. 

To reduce the complexity of the notation we assume that every state has the same number of actions which is denoted by $A := |\mathcal{A}(s)|$ for any $s \in \mathcal{S}$ and we let $m := KA$. Note that having a different number of actions in each state does not affect the validity of this approach but does make the notation more cumbersome. 

We begin by placing restrictions on the learning rates and the MDP. In this algorithm we select learning rates $\{\alpha(n)\}_{n \in \mathbb{N}}$ and $\{\gamma(n)\}_{n \in \mathbb{N}}$ which satisfy (B2) and (B5)(a) and use two asynchronous counters,

\begin{equation*}
	\nu_n(s) := \sum_{i=1}^n \mathbb{I}_{\{s_i=s\}}, \quad \phi_n(s,a) := \sum_{i=1}^n \mathbb{I}_{\{(s_i,a_i)=(s,a)\}}.
\end{equation*}
In addition we assume that the set of transition probabilities, $\{P_{ss'}(a)\}_{s,s',a}$, form an aperiodic, irreducible, positive recurrent Markov chain. Moreover, at every iteration when the MDP is in the state $s$ we enforce that every action $a \in \mathcal{A}(s)$ is played with a non-zero minimum probability and that this holds for every $s \in \mathcal{S}$. Therefore for all $s \in \mathcal{S}$, $a \in \mathcal{A}(s)$ and $n \geq 0$, $\pi_n(s,a) \geq \varepsilon$ for some $\varepsilon > 0$. 

Finally, before directly analysing the algorithm we present a method for verifying the global convergence of a standard differential inclusion in the form of \eqref{DI}.

\begin{LyapunovDfn}
	\label{LyapunovDfn}
	Let $A \subset \mathbb{R}^K$. A continuous function $W: \mathbb{R}^K \mapsto \mathbb{R}$ is a \emph{Lyapunov function} for the differential inclusion \eqref{DI} if it satisfies the following criteria.
	
	\begin{itemize}
		\item[(i)] $W(y) < W(x)$ for all $x \in \mathbb{R}^K \backslash A$, $y \in \Phi_t(x)$, $t > 0$.
		\item[(ii)] $W(y) \leq W(x)$ for all $x \in A$, $y \in \Phi_t(x)$, $t > 0$.
	\end{itemize}
\end{LyapunovDfn}	

Finding a Lyapunov function proves the global convergence of set-valued dynamical systems in the form of \eqref{DI}, a concept which is fully described by Bena{\"{\i}}m, Hofbauer and Sorin \cite{BenaimHofbauerSorin}.

Because we include the asynchronicity with the mean field the associated differential inclusions associated with our algorithm will be in the form,

\begin{equation}
	\dot{x} \in \Omega^{\delta}_k \cdot \Big[ h(x)-x \Big], \label{BasicSetDI}
\end{equation}
where $h(\cdot)$ is set-valued. Let $h'(x) := \Omega^{\delta}_k \cdot [ h(x)-x ]$. We state slight modifications a result of Konda and Borkar \cite{KondaBorkar} and a result of Bena{\"{\i}}m, Hofbauer and Sorin \cite{BenaimHofbauerSorin2} to allow us to easily prove the convergence of differential inclusions in the form of \eqref{BasicSetDI}. We do not state proofs for either of these as they are straightforward extensions of \cite[Lemma 5.4]{KondaBorkar} and \cite[Theorem 3.10]{BenaimHofbauerSorin2}. Let

\begin{equation*}
	K_{s,a}(\pi) := r(s,a) + \beta \sum_{s' \in S} P_{ss'}(a)V^{\pi}(s') - V^{\pi}(s),
\end{equation*}
and $K_s(\pi)$ be the $A$-vector of these terms for all $a \in \mathcal{A}(s)$. For all $s \in \mathcal{S}$, $a \in \mathcal{A}(s)$. In addition, let $\nabla_{\pi} V(s)$ denote taking the partial derivative of $V$ with respect to $\pi$. 

\begin{K-BLyapunov}
	\label{K-BLyapunov}
	Let $G$ be a vector field on $\Delta_K$ conditional on $\pi$. If
	
	\begin{equation*}
		\left\langle G_{s}(\pi), K_{s}(\pi)\right\rangle \geq 0,
	\end{equation*}
	then
	\begin{equation*}
		\left\langle G(\pi), \nabla_{\pi}V^{\pi}(s) \right\rangle \geq \left\langle G_{s}(\pi), K_{s}(\pi)\right\rangle \geq 0.
	\end{equation*}
\end{K-BLyapunov}

\begin{DIConvergence}
	\label{DIConvergence}
	Assuming that $h'(x)$ is a stochastic approximation map and there exists a positive definite function $W \in C^1(\mathbb{R}^k,\mathbb{R})$ such that if $\Lambda = \{W(x)=0; x \in C\}$ and $x \in C \backslash \Lambda$ then for any $\omega \in \Omega^{\delta}_k$, $x' \in h(x)$,
	
	\begin{equation*}
		\left\langle \nabla_x W(x),\omega(x'-x)\right\rangle \leq 0. 
	\end{equation*}
	Then $W(\cdot)$ is a Lyapunov function for \eqref{BasicSetDI} with attracting set $\Lambda$.
\end{DIConvergence}

Using Lemma \ref{K-BLyapunov} and Theorem \ref{DIConvergence} we show the convergence of the following algorithm.

%%%%%%%%%%%%%%%%%%%%%%%%%%%%
%%%%%%%%%%%%%%%%%%%%%%%%%%%%
%%%%%%%%%%%%%%%%%%%%%%%%%%%%
%%%%%%%%%%%%%%%%%%%%%%%%%%%%
%%%%%%%%%%%%%%%%%%%%%%%%%%%%
%%%%%%%%%%%%%%%%%%%%%%%%%%%%

\subsection{The Algorithm}
\label{sec:alg}

This algorithm cannot be studied in the framework of Konda and Borkar \cite{KondaBorkar} due to the Lipschitz continuous restriction they place on the mean fields of the coupled stochastic approximations. In this work we have relaxed this condition allowing the study of process which are based on the best response. Firstly, if $\{Q(s,a)\}_{s,a}$ is the set of action values for a MDP, let $b_s(Q) := \arg\max_{a \in \mathcal{A}(s)} \{Q(s,a)\}$ be the best response set to $\{Q(s,a)\}_{s,a}$ for state $s \in \mathcal{S}$. Use the following coupled algorithm to estimate action values and the optimal strategy for all $s \in \mathcal{S}$ and $a \in \mathcal{A}(s)$,

\begin{align}
	Q_{n+1}\big(s,a\big) & = Q_{n}\big(s,a\big) + \gamma\big(\phi_{n+1}(s,a)\big) \mathbb{I}_{\{(s,a) = (s_{n+1},a_{n+1})\}} \label{Alg2:QUpdates} \\
	& \quad\quad\quad\quad \times \Big[R_{n+1} + \beta V_n(s_{n+2}) - Q_{n}\big(s,a\big)\Big], \notag \\
	\pi_{n+1}(s) & = \pi_n(s) + \mu(\nu_{n+1}(s)) \mathbb{I}_{\{s = s_{n+1}\}} \Big[b_s(Q_{n}) -  \pi_n(s) \Big],	\label{Alg2:PiUpdates}
\end{align}
where $V_n(s) = \sum_{a \in \mathcal{A}(s)} \pi_n(s,a) Q_n(s,a)$. The action $a_n$ is selected using an $\varepsilon$-greedy version of the strategy $\pi_n$. For all $s \in \mathcal{S}$ and $a \in \mathcal{A}(s)$ let

\begin{equation*}
	\pi_n^{\varepsilon}(s,a) := \pi_n(s,a) (1 - A \varepsilon) + \varepsilon.
\end{equation*}
Then $\mathbb{P}(a_{n+1}= a|s_{n+1} =s) = \pi_n^{\varepsilon}(s,a)$. We must verify that (B1)-(B6) hold for this algorithm. We do not directly verify that (B1) holds for this algorithm, but as pointed out in Section \ref{sec:SingleAssumptions} methods to do so are discussed elsewhere. Furthermore the choice of learning parameters verifies (B2) and the choice of mean field in \eqref{Alg2:QUpdates} and \eqref{Alg2:PiUpdates} immediately give that (B3) and (B5) hold. 

For this algorithm we have $\bar{J} = \{(s,a) ; s \in \mathcal{S}, a \in \mathcal{A}(s)\}$ and $\bar{I} = \{s ; s \in \mathcal{S}\}$. This gives that $\bar{H} = \big\{\big((s,a),s\big) ; s \in \mathcal{S}, a \in \mathcal{A}(s)\big\}$ and for simplicity we write,

\begin{equation*}
	\bar{H} = \big\{(s,a) ; s \in \mathcal{S}, a \in \mathcal{A}(s)\big\}.
\end{equation*}
Following the notation of Section \ref{sec:TwoAsynchSA} we have that $z_n = (Q_n,\pi_n)$. With $\mathcal{H}_n,\mathcal{H}_{n+1} \in \bar{H}$ such that $\mathcal{H}_n = (s,a)$, $\mathcal{H}_{n+1} = (s',a')$, then $Q_{(\mathcal{H}_n,\mathcal{H}_{n+1})}(z) = \pi_n^{\varepsilon}(s',a') P_{ss'}(a)$. This shows that (B4)(a) is satisfied. Again using the notation of Section \ref{sec:TwoAsynchSA} the set of transition probabilities are denoted,

\begin{equation*} 
	\Big\{Q_{(\mathcal{H}_n, \mathcal{H}_{n+1})}(z_n)\Big\}_{\mathcal{H}_n,\mathcal{H}_{n+1},z_n} = \Big\{\pi^{\varepsilon}_n(s',a')P_{ss'}(a)\Big\}_{s,s',a,a'}.
\end{equation*}
By assumption on $\{P_{ss'}(a)\}_{s,s',a}$ we have that (B4)(b) is satisfied. Since $\pi^{\varepsilon}_n(s',a') P_{ss'}(a)$ is a continuous function of $\pi^{\varepsilon}_n$, which similarly is a continuous function of $\pi_n \in z_n$, (B4)(c) is satisfied. 

A consequence of (B4) from Appendix \ref{MarkovUpdates} is that in the limit every state of the MDP is visited a minimum proportion of the time, $\eta$, for some $\eta > 0$. Similarly, by placing the restriction that every action is selected with at least probability $\varepsilon$ for some $\varepsilon > 0$ then every state, action pair is taken a minimum proportion of the time, $\eta'$, for some $\eta' > 0$. Using the approach of Section \ref{sec:TwoAsynchSA} we do not explicitly need to know the values $\eta$ and $\eta'$ as we verify convergence for every $\eta, \eta' > 0$.

Finally, we need to verify (B6). Define 

\begin{equation*}
	Q^{\pi}(s,a) := r(s,a) + \beta \sum_{s' \in \mathcal{S}} P_{ss'}(a) V^{\pi}(s),
\end{equation*}
and let $Q^{\pi}(s)$ be the $A$-vector containing $Q^{\pi}(s,a)$ for all $a \in \mathcal{A}(s)$. Let $h(\cdot,\cdot): \Delta_K \times \mathbb{R}^m \rightarrow \mathbb{R}^m$ be defined such that,

\begin{equation*}
	h_{s,a}(\pi,Q) :=  r(s,a) + \beta \sum_{s' \in \mathcal{S}} P_{ss'}(a) V(s'),
\end{equation*}
where $V(s) = \sum_{a \in \mathcal{A}(s)} \pi(s,a)Q(s,a)$. Let $h_s(\pi,Q)$ be the $A$-vector of the $h_{s,a}(\pi,Q)$ terms which means that $h(\pi,Q^{\pi}) = Q^{\pi}$. For fixed $\pi \in \Delta_K$ consider the differential inclusion 

\begin{equation}
	\dot{Q}_{s}(t) \in \Omega_{A}^{\eta'} \cdot \Big[h_s\big(\pi,Q_{s}(t)\big) - Q_{s}(t)\Big], \quad \mbox{for all } s \in \mathcal{S}, \label{Alg2:QInclusion}
\end{equation}
\begin{algorithm2QValueConvvergence}
	\label{algorithm2QValueConvvergence}
	$Q^{\pi}(s)$ is the unique asymptotically stable equilibrium to \eqref{Alg2:QInclusion}.	
\end{algorithm2QValueConvvergence}

\begin{proof}
	$h_s\big(\pi,Q_{s}(t)\big)$ is a contraction mapping \cite{SzepesvariLittman}, \cite{Tsitsiklis}. Hence for any fixed $\omega \in \Omega_{A}^{\eta'}$, $Q_{s}(t) \rightarrow Q^{\pi}(s)$. Combining this with the note by Borkar \cite[Chapter 7.4]{BorkarBook} proves the claim. 
\end{proof}

From this it follows that the values in $\{Q_n(s,a)\}_{s,a}$ converge to the true action values for the strategy $\pi$. These values are Lipschitz continuous in $\pi$, which ensures (B6) holds. Hence, Theorem \ref{xbarJointlyPerturbedSol} holds and the linear interpolation of the iterative process in \eqref{Alg2:PiUpdates} is an asymptotic pseudo-trajectory to the differential inclusion

\begin{equation}
	\dot{\pi}_s(t) \in \Omega^{\eta} \cdot \Big[b_s(Q^{\pi(t)}) - \pi_s(t)\Big], \quad \mbox{for all } s \in \mathcal{S}, \label{Alg2:PiInclusion}
\end{equation}
for some $\eta >0$. For a particular action $a \in \mathcal{A}(s)$, let $\pi_{s,a}(t)$ and $\dot{\pi}_{s,a}(t)$ represent the individual components of $\pi_s(t)$ and $\dot{\pi}_s(t)$ respectively, whilst $\pi(t)$ and $\dot{\pi}(t)$ are the $K \times A$ matrices containing all of the $\pi_{s,a}(t)$ and $\dot{\pi}_{s,a}(t)$ elements.

With assumptions (B1)-(B6) satisfied all that remains is to show that the differential inclusion \eqref{Alg2:PiInclusion} has a globally attracting set. Corollary \ref{twotimescaleconvergence} will then provide the convergence result of the coupled algorithm in \eqref{Alg2:QUpdates} and \eqref{Alg2:PiUpdates}. We note that \eqref{Alg2:PiInclusion} is in the form of \eqref{BasicSetDI} and hence we use Theorem \ref{DIConvergence} to prove the global convergence.

\begin{algorithm2BRConvergence}
	\label{algorithm2BRConvergence}
	For each $s \in \mathcal{S}$ fix an $\omega_s \in \Omega^{\eta}$. Take any strategy $\underline{\pi} \in \Delta_K$ and for each $s \in \mathcal{S}$ select $\underline{\tilde{\pi}}_s \in b_s(Q^{\underline{\pi}})$. With
	
	\begin{equation*}
		\underline{\dot{\pi}}_s := \omega_s \big[\underline{\tilde{\pi}}_s - \underline{\pi}_s\big], \quad \mbox{for all } s \in \mathcal{S}.
	\end{equation*}
	Then for any $s \in \mathcal{S}$,
		
	\begin{equation*}
		\left\langle \nabla_{\underline{\pi}} V^{\underline{\pi}}(s), \underline{\dot{\pi}}\right\rangle \geq 0.
	\end{equation*}
\end{algorithm2BRConvergence}

\begin{proof}
	Fix a strategy $\underline{\pi} \in \Delta_K$ and for all $s \in \mathcal{S}$ fix $\omega_s \in \Omega^{\eta}$ and take $\underline{\tilde{\pi}}_s \in b_s(Q^{\underline{\pi}})$. Let 
	
	\begin{equation*}
		\underline{\dot{\pi}}_s = \omega_s \big[\underline{\tilde{\pi}}_s - \underline{\pi}_s\big], \mbox{for all } s \in \mathcal{S}.
	\end{equation*}
	Consider,
	
	\begin{align*}
		\left\langle \underline{\dot{\pi}}_s, K_{s}(\underline{\pi})\right\rangle & = \omega_s \left[\sum_{a \in \mathcal{A}(s)} \underline{\tilde{\pi}}_s(a) Q^{\underline{\pi}}(s,a) - \sum_{a \in \mathcal{A}(s)} \underline{\pi}_s(a) Q^{\underline{\pi}}(s,a) \right] \\
		& \quad - \omega_s \left[\sum_{a \in \mathcal{A}(s)} \underline{\tilde{\pi}}_s(a) V^{\underline{\pi}}(s) - \sum_{a \in \mathcal{A}(s)} \underline{\pi}_s(a) V^{\underline{\pi}}(s) \right].
	\end{align*}
	The second term here is zero since $\sum_{a \in \mathcal{A}(s)} \rho(s,a)V^{\pi}(s) = V^{\pi}(s)$ for any $\rho(s) \in \Delta(\mathcal{A}(s))$. The first term is clearly positive by the definition of the best response. Hence
	
	\begin{equation*}
		\left\langle \underline{\dot{\pi}}_s, K_{s}(\underline{\pi})\right\rangle \geq 0.
	\end{equation*}
	Then using Lemma \ref{K-BLyapunov} gives the desired result.
\end{proof}

Now we use Theorem \ref{DIConvergence} and Lemma \ref{algorithm2BRConvergence} to produce a Lyapunov function for \eqref{Alg2:PiInclusion} and hence prove the global convergence of the second algorithm given by \eqref{Alg2:QUpdates} and \eqref{Alg2:PiUpdates}. Let $\Lambda := \{\pi; \pi \mbox{ an optimal strategy}\}$.

\begin{algorithm2convergence}
	\label{algorithm2convergence}
	Fix $\tilde{\pi}$ as an optimal strategy. Then
	
	\begin{equation*}	
		W(\pi) = \sum_{s \in \mathcal{S}} \Big[V^{\tilde{\pi}}(s) - V^{\pi}(s)\Big],
	\end{equation*}	
	is a Lyapunov function for the differential inclusion \eqref{Alg2:PiInclusion} and $\Lambda$ is a globally attracting set for \eqref{Alg2:PiInclusion}.
\end{algorithm2convergence}

\begin{proof}
	We prove the claim by applying Theorem \ref{DIConvergence} to \eqref{Alg2:PiInclusion}. Clearly $W(\cdot)$ is positive semi-definite since for all $s \in S$, and any $\pi \in \Delta_K$, $V^{\tilde{\pi}}(s) \geq V^{\pi}(s)$, with equality for $\pi \in \Lambda$. 
	
	To prove the condition of Theorem \ref{DIConvergence} we note that for any fixed $t >0$,
	
	\begin{equation*}
		\left\langle \nabla_{\pi(t)} W\big(\pi(t)\big), \dot{\pi}(t)\right\rangle = - \sum_{s \in \mathcal{S}} \left\langle \nabla_{\underline{\pi}} V^{\underline{\pi}}(s), \underline{\dot{\pi}}\right\rangle, %\label{Alg2:LyapunovDeriv}
	\end{equation*}
	for some strategy $\underline{\pi} \in \Delta_K$, and for each $s \in \mathcal{S}$, a fixed $\omega_s \in \Omega^{\eta}$ and $\underline{\tilde{\pi}}_s \in b_s(Q^{\underline{\pi}})$ such that
	
	\begin{equation*}
		\underline{\dot{\pi}}_s = \omega_s \big[\underline{\tilde{\pi}}_s - \underline{\pi}_s\big], \mbox{for all } s \in \mathcal{S}.
	\end{equation*}
	Using Lemma \ref{algorithm2BRConvergence} immediately gives that $\left\langle \nabla_{\pi(t)} W\big(\pi(t)\big), \dot{\pi}(t)\right\rangle < 0$ for all $\pi(t) \in \Delta_k \backslash \Lambda$ and for any $t >0$. Applying Theorem \ref{DIConvergence} proves the claim.
\end{proof}
	
\begin{algorithm2Corollary}
	\label{algorithm2Corollary}
	The coupled process $(Q_n, \pi_n)$ from \eqref{Alg2:QUpdates} and \eqref{Alg2:PiUpdates} converges to the limit $(Q^{\tilde{\pi}}, \tilde{\pi})$, where $\tilde{\pi}$ is an optimal strategy and $\{Q^{\tilde{\pi}}(s,a)\}_{s,a}$ is the set of associated action values.
\end{algorithm2Corollary}

\begin{proof}
	Lemma \ref{algorithm2convergence} shows that a Lyapunov function exists for the differential inclusion \eqref{Alg2:PiInclusion}; this with Corollary \ref{twotimescaleconvergence} proves the claim.
\end{proof}

%%%%%%%%%%%%%%%%%%%%%%%%%%%%
%%%%%%%%%%%%%%%%%%%%%%%%%%%%
%%%%%%%%%%%%%%%%%%%%%%%%%%%%
%%%%%%%%%%%%%%%%%%%%%%%%%%%%
%%%%%%%%%%%%%%%%%%%%%%%%%%%%
%%%%%%%%%%%%%%%%%%%%%%%%%%%%

\section{Summary}

We have combined the work of Bena{\"{\i}}m, Hofbauer and Sorin \cite{BenaimHofbauerSorin} on differential inclusions with the work of Borkar and Konda \cite{KondaBorkar} and Borkar \cite{BorkarAsynchronous} on asynchronous stochastic approximation in order to provide a framework for asynchronous stochastic approximation with a set-valued mean field. This enables us to modify the previous work on asynchronous stochastic approximation to use a set of assumptions which are straightforward to verify \emph{a priori}. 

Furthermore we extended the work of Konda and Borkar on asynchronous two-timescale stochastic approximation using this new framework. By allowing the mean fields to be updated using set-valued functions we provide a new result in two-timescale stochastic approximations which clearly applies to asynchronous and synchronous stochastic approximations.

This approach provides a clear framework for single or multiple timescale asynchronous stochastic approximations with clear assumptions which differ little from the synchronous case. Where previously the additional and difficult to verify assumptions required for asynchronous stochastic approximations could be perceived as a reason to avoid their use, this framework removes many of these issues.

In Section \ref{sec:MDPApplication} we provided an example of a coupled learning algorithm for a discounted reward Markov decision process. We analysed the limiting behaviour using the results of Section \ref{sec:TwoAsynchSA}. The algorithm uses a set-valued mean field based upon a best response style of actor-critic learning. We used the results of Section \ref{sec:TwoAsynchSA} to show convergence to an optimal strategy under a straightforward set of assumptions. This algorithm demonstrates the main value of the approach to asynchronous stochastic approximation in this paper.

%%%%%%%%%%%%%%%%%%%%%%%%%%%%
%%%%%%%%%%%%%%%%%%%%%%%%%%%%
%%%%%%%%%%%%%%%%%%%%%%%%%%%%
%%%%%%%%%%%%%%%%%%%%%%%%%%%%
%%%%%%%%%%%%%%%%%%%%%%%%%%%%
%%%%%%%%%%%%%%%%%%%%%%%%%%%%

\appendix
\section{Omitted Proofs}

%%%%%%%%%%%%%%%%%%%%%%%%%%%%
%%%%%%%%%%%%%%%%%%%%%%%%%%%%
%%%%%%%%%%%%%%%%%%%%%%%%%%%%
%%%%%%%%%%%%%%%%%%%%%%%%%%%%
%%%%%%%%%%%%%%%%%%%%%%%%%%%%
%%%%%%%%%%%%%%%%%%%%%%%%%%%%

\subsection{Minimum Update Proportion}
\label{MarkovUpdates}

For the asynchronous stochastic approximations in \eqref{asynchronousSAmatrix} we are interested in understanding how different components of $x_n$ get selected to be updated. The previous work on this makes the direct assumption that in the limit all the elements of $\bar{I}$ are updated in an equally spaced manner in the limit and some minimum proportion of the iterations \cite{KondaBorkar}; this assumption is difficult to verify prior to running the process. In this work we use results on Markov chains via assumption (A4) as an alternative which can be checked \itshape a priori \upshape if we know the transition probabilities of the state process. 

\begin{minupdateproportion}
	\label{minupdateproportion}
	Under (A4), there exists $\eta > 0$ such that $\forall i \in I$, 
	\begin{equation*}
		\lim\inf_{n \rightarrow \infty} \frac{\nu_n(i)}{n} \geq \eta, \quad \mbox{a.s.}.
	\end{equation*}
\end{minupdateproportion}

\begin{proof}
	The values of $\bar{I}_{n} \in \bar{I}$ form a controlled Markov chain where $\bar{I}_{n+1}$ depends on the current updated component, $\bar{I}_n$, and the value of the iterative process, $x_n$. For $x \in C$ let $\pi_x(\cdot)$ be a stationary distribution for this Markov chain given by the transition probabilities $P_x(\cdot, \cdot)$; standard theory gives that, under (A4), $\pi_x(\cdot)$ exists, is unique, and for some $\delta_x >0$, $\pi_x(\mathcal{I}) > \delta_x$ for all $\mathcal{I} \in \bar{I}$. Let $\eta = \min_{x \in C} \delta_x$, which exists and is positive since $C$ is compact. Then for all $\mathcal{I} \in \bar{I}$, $x \in C$ we have that 

	\begin{equation}
		\pi_x(\mathcal{I}) > \delta_x \geq \eta. \label{stationaryprob}
	\end{equation}
	For $\mathcal{I} \in \bar{I}$ define
	
	\begin{equation*}
		w_n(\mathcal{I}) = \sum_{k=1}^n \mathbb{I}_{\{ \bar{I}_k = \mathcal{I} \}},
	\end{equation*}
	And let $\bar{w}(\mathcal{I}) = w_n(\mathcal{I})/n$.
	
	\begin{align}
		\bar{w}_n(\mathcal{I}) & = \frac{w_{n-1}(\mathcal{I})}{n} + \frac{1}{n}\mathbb{I}_{\{\bar{I}_n = \mathcal{I}\}} \notag \\
		& = \bar{w}_{n-1}(\mathcal{I}) + \frac{1}{n} \Big(\mathbb{I}_{\{\bar{I}_n = \mathcal{I}\}} - \bar{w}_{n-1}(\mathcal{I}) \Big) \label{MarkovSA}
	\end{align}
	This is in the form of a stochastic approximation with controlled Markovian noise as in \cite{BorkarMarkovNoise}.
	
	Let $\bar{w}(t)$ be the linear interpolation of the $\{\bar{w}_n\}_{n \in \mathbb{N}}$ process. Using \cite[Corollary 3.1]{BorkarMarkovNoise} the limiting behaviour of the interpolated process, $\bar{w}(t)$, will be an asymptotic pseudo-trajectory to a differential equation,
	
	\begin{equation}
		\frac{d w(t)}{dt} = \pi_{x(t)} - w(t). \label{markovnoiseDE}
	\end{equation}	
	For a suitable process $x(t) \in C$ based upon $\{x_n\}_{n \in \mathbb{N}}$.	We know that $\pi_{x}(\mathcal{I}) \geq \eta > 0$	for all $n$ and hence the dynamics of $w(t)$ can be expressed as,
	
	\begin{equation*}
		\frac{d w(t)}{dt} \in \Omega^{\eta}_{|\bar{I}|} - w(t),
	\end{equation*}	
	where $\Omega^{\eta}_{|\bar{I}|}$ is as defined in \eqref{OmegaSetdfn}. This implies that any limit point of $w(t)$ will be in $\Omega^{\eta}_{|\bar{I}|}$ and hence $\lim\inf_{n \rightarrow \infty} \frac{w_n(\mathcal{I})}{n} \geq \eta$ a.s. $\forall \mathcal{I} \in \bar{I}$. 
	
	For $i \in I$, define $\mathcal{I}(i) := \{\mathcal{I} \in \bar{I}; i \in \mathcal{I}\}$ then $\nu_n(i) = \sum_{\mathcal{I} \in \mathcal{I}(i)} w_n(\mathcal{I})$, and so for some $\mathcal{I} \in \bar{I}$,
	
	\begin{equation*}
		\lim\inf_{n \rightarrow \infty} \frac{\nu_n(i)}{n} \geq \lim\inf_{n \rightarrow \infty} \frac{w_n(\mathcal{I})}{n} \geq \eta, \quad \mbox{a.s.}.  \qedhere
	\end{equation*}
\end{proof}

%%%%%%%%%%%%%%%%%%%%%%%%%%%%
%%%%%%%%%%%%%%%%%%%%%%%%%%%%
%%%%%%%%%%%%%%%%%%%%%%%%%%%%
%%%%%%%%%%%%%%%%%%%%%%%%%%%%
%%%%%%%%%%%%%%%%%%%%%%%%%%%%
%%%%%%%%%%%%%%%%%%%%%%%%%%%%

\subsection{Noise Conditions for an Asynchronous Stochastic Approximation}
\label{Appendix:Noise}

Let $\tau_0 = 0$, $\tau_n = \sum_{k=1}^{n} \alpha(k)$ and recall $\bar{\tau}_0 = 0$, $\bar{\tau}_n = \sum_{k=1}^{n} \bar{\alpha}_k$. Denote the \itshape asynchronous noise \upshape term $\hat{V}_n(i) = \frac{\alpha(\nu_n(i))}{\alpha(n)} \mathbb{I}_{\{i \in \bar{I}_n\} }V_n(i)$, and let $\hat{V}_n$ be the $K$-vector of these terms.

\begin{Benaimconditions}
	\label{Benaimconditions}
	If $\{V_n\}_{n \in \mathbb{N}}$ is a martingale difference noise process and assuming (A2)(b) and (A4) hold;
	\leavevmode
	\begin{itemize}
		\item[(i)] If $\sup_n \mathbb{E} \Big[\|V_n\|^q\Big] < \infty$ for some $q >0$ then almost surely,
	\begin{equation*}
		\sup_n \mathbb{E} \Big[\|\hat{V}_n\|^q\Big] < \infty.
	\end{equation*}
	\end{itemize}

	\begin{itemize}
		\item[(ii)] If $V_n(i)$ independent of $\bar{I}_n$ given $\mathcal{F}_{n-1}$ and $V_n(i)$ independent of $V_n(j)$ for $i \neq j$. Let $\langle a,b \rangle = \sum_k a_k b_k$. If there exists a $\Gamma > 0$ such that for all $\theta \in \mathbb{R}^K$,
		\begin{equation*}
			\mathbb{E}\Big[\exp\Big\{\langle \theta,V_{n+1} \rangle \Big\} | \mathcal{F}_n\Big] \leq \exp\Big\{\frac{\Gamma}{2} \|\theta\|^2 \Big\}
		\end{equation*}
		then almost surely there exists a $\bar{\Gamma} > 0$ such that for all $\theta \in \mathbb{R}^K$,
	\begin{equation*}
		\mathbb{E}\Big[\exp\Big\{\langle \theta,\hat{V}_{n+1} \rangle \Big\} | \mathcal{F}_n\Big] \leq \exp\Big\{\frac{\bar{\Gamma}}{2} \|\theta\|^2 \Big\}.
	\end{equation*}
	\end{itemize}
\end{Benaimconditions}

\begin{proof}
	\begin{itemize}
		\item[(i)] Combine $\eta$ from Lemma \ref{minupdateproportion} with (A2)(b) to give $A_{\eta} \geq 1$. Then $\|\hat{V}_n(i)\| \leq \|A_{\eta} V_n(i)\|$ using (A2)(b). From this it is immediate that
		
		\begin{equation*}
			\mathbb{E} \Big[\|\hat{V}_n(j)\|^q\Big] \leq A_{\eta}^q \mathbb{E} \Big[\|V_n\|^q\Big] < \infty.
		\end{equation*}
		\item[(ii)] Let $\check{V}_n(i) := \mathbb{I}_{\{i \in \bar{I}_n\} }V_n(i)$. Clearly $\|\check{V}_n\| \leq \|V_n\|$. In addition let $e_i$ be a $K$ dimensional basis vector with a 1 in the $i^{\mathrm{th}}$ term and $0$ everywhere else. From the assumption in lemma \ref{Benaimconditions},
		
		\begin{equation*}
		\mathbb{E} \bigg[ \exp \Big\{ \big< \theta_i e_i, \check{V}_{n+1} \big> \Big\} | \mathcal{F}_n \bigg] \leq \exp\Big\{\frac{\Gamma}{2} \theta_i^2 \Big\}.
		\end{equation*}
		Using this gives the following,
		
		\begin{align*}
			\mathbb{E} \bigg[ \exp \Big\{ \big< \theta, \check{V}_{n+1} \big> \Big\} | \mathcal{F}_n  \bigg] & = \mathbb{E} \bigg[ \exp \Big\{ \big< \sum_{i=1}^K \theta_i e_i, \check{V}_{n+1} \big> \Big\} | \mathcal{F}_n  \bigg], \\
			& = \prod_{i=1}^K \mathbb{E} \bigg[ \exp \Big\{ \big<\theta_i e_i, \check{V}_{n+1} \big> \Big\} | \mathcal{F}_n  \bigg], \\
			& \leq \prod_{i=1}^K \exp\Big\{\frac{\Gamma}{2} \theta_i^2 \Big\}, \\
			& = \exp\Big\{\frac{\Gamma}{2} \| \theta \|^2 \Big\}.
		\end{align*}
		 Notice that $\frac{\alpha(\nu_{n+1}(i))}{\alpha(n+1)} \mathbb{I}_{\{i \in \bar{I}_{n+1} \}} = \frac{\alpha(\nu_{n}(i) + 1)}{\alpha(n+1)} \mathbb{I}_{\{i \in \bar{I}_{n+1} \}}$ and $\frac{\alpha(\nu_{n}(i) + 1)}{\alpha(n+1)}$ is deterministic given $\mathcal{F}_n$.
		 
		\begin{align*}
			\mathbb{E} \bigg[ \exp \Big\{ \big< \theta, \hat{V}_{n+1} \big> \Big\} | \mathcal{F}_n  \bigg] & = \mathbb{E} \bigg[ \exp \Big\{ \big< \theta,  \sum_{i=1}^K \frac{\alpha(\nu_{n+1}(i))}{\alpha(n+1)} \check{V}_{n+1}(i) e_i \big> \Big\} | \mathcal{F}_n  \bigg], \\
			& = \mathbb{E} \bigg[  \prod_{i=1}^K  \exp \Big\{ \frac{\alpha(\nu_{n}(i)+1)}{\alpha(n+1)} \big< \theta, \check{V}_{n+1}(i) e_i \big> \Big\}| \mathcal{F}_n \bigg], \\
		\end{align*}
		Now using the independence of the $V_i$ terms and letting $\bar{\theta}_n(i) := \frac{\alpha(\nu_{n}(i)+1)}{\alpha(n+1)} \theta$.
			
		\begin{align*}
			\mathbb{E} \bigg[ \exp \Big\{ \big< \theta, \hat{V}_{n+1} \big> \Big\} | \mathcal{F}_n  \bigg] & = \prod_{i=1}^K \mathbb{E} \bigg[ \exp \Big\{  \big< \bar{\theta}_n(i), \check{V}_{n+1}(i) e_i \big> \Big\}| \mathcal{F}_n \bigg], \\
			& = \prod_{i=1}^K \mathbb{E} \bigg[ \exp \Big\{  \big< e_i \bar{\theta}_n(i), \check{V}_{n+1} \big> \Big\}| \mathcal{F}_n \bigg].
		\end{align*}
		Note that $\| \bar{\theta}_n(i) \|^2 \leq \| A_{\eta} \theta \|^2$, with $A_{\eta}$ taken from (A2)(b). Finally, this gives,
	
		\begin{align*}
			\mathbb{E} \bigg[ \exp \Big\{ \big< \theta, \hat{V}_{n+1} \big> \Big\} | \mathcal{F}_n  \bigg] & \leq \prod_{i=1}^K \bigg( \exp\Big\{\frac{\Gamma}{2} \| \bar{\theta}_n(i) \|^2 \Big\} \bigg), \\
			& \leq \prod_{i=1}^K \bigg( \exp\Big\{\frac{\Gamma A_{\eta}^2}{2} \| \theta \|^2 \Big\} \bigg). 
		\end{align*}
		Letting $\bar{\Gamma} := K \Gamma A_{\eta}^2$, which is constant, completes the proof of $(ii)$.	
	\end{itemize}
\end{proof}

\begin{proof}
	(of Lemma \ref{lemma:Noise})

	Firstly, define $\xi_n(t) = \sup \{k\geq0 ; t \geq \tau_{n+k}\}$ and let $\bar{\xi}_n(t) = \sup \{k\geq0 ; t \geq \bar{\tau}_{n+k}\}$, and notice $\xi_n(t) \geq \bar{\xi}_n(t)$ since $\bar{\alpha}_n \geq \alpha(n)$ from (A2)(b). Therefore,
	
	\begin{equation*}
		m(\tau_n + T) = n + \xi_n(T) \geq n + \bar{\xi}_n(T) = \bar{m}(\bar{\tau}_n + T).
	\end{equation*}
	Using this in the following 
	
	\begin{align*}
	 	& \sup_k \bigg\{ \Big\| \sum_{j=n}^{k-1} \bar{\alpha}_{j+1} M_{j+1} V_{j+1} \Big\| ; k = n+1, \dots,\bar{m}(\bar{\tau}_n + T) \bigg\}, \\
		& \quad = \sup_k \bigg\{ \Big\| \sum_{j=n}^{k-1} \alpha(j+1) \hat{V}_{j+1} \Big\| ; k = n+1, \dots,\bar{m}(\bar{\tau}_n + T) \bigg\}, \\
		& \quad \leq \sup_k \bigg\{ \Big\| \sum_{j=n}^{k-1} \alpha(j+1) \hat{V}_{j+1} \Big\| ; k = n+1, \dots,m(\tau_n + T) \bigg\}.
	\end{align*} 

	Combining Lemma \ref{Benaimconditions} with \cite[Proposition 1.4]{BenaimHofbauerSorin} with the martingale noise sequence $\{\hat{V}_n\}_{n \in \mathbb{N}}$ and step sizes $\{\alpha(n)\}_{n \in \mathbb{N}}$ gives that this latter term converges to zero as $t \rightarrow \infty$.
\end{proof}

%%%%%%%%%%%%%%%%%%%%%%%%%%%%
%%%%%%%%%%%%%%%%%%%%%%%%%%%%
%%%%%%%%%%%%%%%%%%%%%%%%%%%%
%%%%%%%%%%%%%%%%%%%%%%%%%%%%
%%%%%%%%%%%%%%%%%%%%%%%%%%%%
%%%%%%%%%%%%%%%%%%%%%%%%%%%%

\subsection{Weak Convergence of the Relative Step Sizes}
\label{appendix:uvconvergence}

Firstly, we will need a result which follows from (A2)(b) and Lemma \ref{minupdateproportion},

\begin{alpharatio}
	\label{alpharatio}
	Under assumptions (A2)(b) and (A4) 
	
	\begin{equation}
		\lim_{n \rightarrow \infty} \inf \left( \frac{\alpha(n)}{\bar{\alpha}_n}\right) \geq A^{-1}_{\eta}, \quad \mbox{a.s.}. \label{alphamin}
	\end{equation}
\end{alpharatio}
	
\begin{proof}
	\begin{align*}
		\lim_{n \rightarrow \infty} \inf \left( \frac{\alpha(n)}{\bar{\alpha}_n}\right) & = \lim_{n \rightarrow \infty} \inf \left( \frac{\alpha(n)}{\alpha(\nu_n(i))}\right), \\
		& \geq \lim_{n \rightarrow \infty} \inf \left( \frac{\alpha(n)}{\alpha(n \eta)}\right), \quad \mbox{a.s.},\\
		& \geq A^{-1}_{\eta}.
	\end{align*}
	Where the first step must hold for some  $i \in I$, the second step follows from Lemma \ref{minupdateproportion} and the last step is directly from (A2)(b).
\end{proof}

\begin{proof}
	(of Lemma \ref{weakconvergence})
	
	Since $L^2([0,T])$ is a Hilbert space it is relatively compact and relatively sequentially compact using the Banach-Alaoglu Theorem \cite[Appendix A]{BorkarBook}, which guarantees that the sequence $\{\bar{u}_i^n(\cdot)\}_{n \in \mathbb{N}}$ has a weakly convergent subsequence with a limit point in $L^2([0,T])$. Hence, there exists a limit point $\tilde{u}_i(\cdot)$ of \eqref{weakconvergencecriteria}. For a fixed $T > 0$, $\tilde{u}_i(\cdot)$ must satisfy \eqref{Dfn:weakconvergence} for all $h(\cdot) \in L^2([0,T])$. Hence by showing that for an arbitrary fixed $T$ and a single $h(\cdot)$ any limit point is bounded below this is enough to prove the claim. Fix $T > 0$, select $0 < v \leq T$, and take $h(t) = 1$ for all $t \in [0,T]$. Let $\{k(n)\}_{n \in \mathbb{N}}$ be a subsequence of the natural numbers under which $\{\bar{u}_i^{k(n)}(\cdot)\}_{n \in \mathbb{N}}$ converges.
	
	\begin{align*}
		\int_0^v \tilde{u}_i(s) \d s & = \lim_{n\rightarrow \infty} \int_0^v u_i(\bar{\tau}_{k(n)} + s) \d s \\
		& = \lim_{n\rightarrow \infty} \int_{\bar{\tau}_{k(n)}}^{\bar{\tau}_{k(n)} + v} u_i(s) \d s \\
		& \geq \lim_{n \rightarrow \infty} \sum_{j= k(n)}^{\bar{m}(\bar{\tau}_{k(n)} + v) - 1} \bar{\alpha}_{j+1} \mu_{j+1}(i), \\
		& = \lim_{n \rightarrow \infty} \sum_{j= k(n)}^{\bar{m}(\bar{\tau}_{k(n)} + v) - 1} \alpha(j+1) \mathbb{I}_{\{ i \in \bar{I}_{j+1}\} }.
	\end{align*}
	The following part of the proof uses a slight modification to the result by Ma et al. \cite[Theorem 2.2]{MaMakowskiShwartz} combined with the stochastic approximation form of the updates $w_n$ from Lemma \ref{minupdateproportion} given in \eqref{MarkovSA}. The modification comes because the transition probabilities for the Markov chain on $\bar{I}$ depends on $x_n$ instead of $w_n$. This requires only a straightforward modification to the proofs in Sections 4 and 5 of \cite{MaMakowskiShwartz}. Using this modification of \cite[Theorem 2.2]{MaMakowskiShwartz} gives,
	
	\begin{align*}
		\lim_{n \rightarrow \infty} \sum_{j= k(n)}^{\bar{m}(\bar{\tau}_{k(n)} + v) - 1} & \alpha(j+1) \mathbb{I}_{\{ i \in \bar{I}_{j+1}\}} \\
		& = \lim_{n \rightarrow \infty} \sum_{\mathcal{I} \in \mathcal{I}(i)} \sum_{j= k(n)}^{\bar{m}(\bar{\tau}_{k(n)} + v) - 1} \alpha(j+1) \mathbb{I}_{\{ \bar{I}_{j+1} = \mathcal{I}\}}, \\
		& = \lim_{n \rightarrow \infty} \sum_{\mathcal{I} \in \mathcal{I}(i)} \sum_{j= k(n)}^{\bar{m}(\bar{\tau}_{k(n)} + v) - 1} \alpha(j+1) \pi_{x_j}(\mathcal{I}).
	\end{align*}
	Now we combine the above with \eqref{stationaryprob} to give 
	
	\begin{align*}
		\lim_{n \rightarrow \infty} \sum_{\mathcal{I} \in \mathcal{I}(i)} \sum_{j= k(n)}^{\bar{m}(\bar{\tau}_{k(n)} + v) - 1} & \alpha(j+1) \pi_{x_j}(\mathcal{I}) \\
		& \geq \lim_{n \rightarrow \infty} \sum_{\mathcal{I} \in \mathcal{I}(i)} \sum_{j= k(n)}^{\bar{m}(\bar{\tau}_{k(n)} + v) - 1} \alpha(j+1) \eta, \\
		& \geq \lim_{n \rightarrow \infty} \sum_{j= k(n)}^{\bar{m}(\bar{\tau}_{k(n)} + v) - 1} \bar{\alpha}_{j+1} \frac{\alpha(j+1)}{\bar{\alpha}_{j+1}} \eta.
	\end{align*}
	Now using Lemma \ref{alpharatio},
	
	\begin{align*}
		\lim_{n \rightarrow \infty} \sum_{j= k(n)}^{\bar{m}(\bar{\tau}_{k(n)} + v) - 1} & \bar{\alpha}_{j+1} \frac{\alpha(j+1)}{\bar{\alpha}_{j+1}} \eta \\
		& \geq \lim_{n \rightarrow \infty} \sum_{j= k(n)}^{\bar{m}(\bar{\tau}_{k(n)} + v) - 1} \bar{\alpha}_{j+1} A^{-1}_{\eta} \eta, \quad \mbox{a.s.}, \\
	\end{align*}
	We convert the sum back to an integral to give,
		
	\begin{align*}
		\lim_{n \rightarrow \infty} \sum_{j= k(n)}^{\bar{m}(\bar{\tau}_{k(n)} + v) - 1} & \bar{\alpha}_{j+1} A^{-1}_{\eta} \eta \\
		& \geq \lim_{n \rightarrow \infty} \int_0^v A^{-1}_{\eta} \eta \d s - \lim_{n \rightarrow \infty} \bar{\alpha}_{\bar{m}(\bar{\tau}_{k(n)} + v) + 1} A^{-1}_{\eta} \eta, \\
		& = \int_0^v A^{-1}_{\eta} \eta \d s. 
	\end{align*}
	Taking $\varepsilon = A^{-1}_{\eta} \eta$ completes the proof.
\end{proof}

We shall now prove the important corollary to Lemma \ref{weakconvergence}.

\begin{proof}
	(Of Corollary \ref{weakconvergencecorol1})

	Begin by noting that 
	
	\begin{align*}
		\int_t^{t + v} \tilde{u}_i(s) \d s & = \lim_{n \rightarrow \infty} \int_t^{t + v} u_i^{k(n)}(s) \d s, \\
		& = \lim_{n \rightarrow \infty} \int_{t + \bar{\tau}_{k(n)}}^{t + \bar{\tau}_{k(n)} + v} u_i(s) \d s.
	\end{align*}
	Now let $\bar{n} = \bar{m}(t + \bar{\tau}_{k(n)})$ and note that $u_i(t) \in [0,1]$.
	
	\begin{align*}
		\lim_{n \rightarrow \infty} \int_{t + \bar{\tau}_{k(n)}}^{t + \bar{\tau}_{k(n)} + v} u_i(s) \d s & \geq \lim_{\bar{n} \rightarrow \infty} \bigg[ \int_{\bar{\tau}_{\bar{n}}}^{\bar{\tau}_{\bar{n}} + v} u_i(s) \d s - \bar{\alpha}_{\bar{n} + 1} \bigg], \\
		& = \lim_{\bar{n} \rightarrow \infty} \int_0^v u_i^{\bar{n}}(s) \d s.
	\end{align*}
	Now, $\mathcal{U}$ is a compact metric space \cite[Appendix A]{BorkarBook} and hence is weakly countably (limit point) compact. This means the infinite subset $\{u_i^{\bar{n}}(\cdot)\}_{n \in \mathbb{N}}$ has a limit in $\mathcal{U}$, and this limit point must satisfy Lemma \ref{weakconvergence}. This gives,
	
	\begin{equation*}
		\int_{\bar{t}}^{\bar{t} + v} \tilde{u}_i(s) \d s \geq \lim_{\bar{n} \rightarrow \infty} \int_0^{v} u_i^{\bar{n}}(s) \d s \geq v \varepsilon, \quad \mbox{a.s.},
	\end{equation*}
	and since this statement is true for all $t,v>0$, the statement follows.
\end{proof}

%%%%%%%%%%%%%%%%%%%%%%%%%%%%
%%%%%%%%%%%%%%%%%%%%%%%%%%%%
%%%%%%%%%%%%%%%%%%%%%%%%%%%%
%%%%%%%%%%%%%%%%%%%%%%%%%%%%
%%%%%%%%%%%%%%%%%%%%%%%%%%%%
%%%%%%%%%%%%%%%%%%%%%%%%%%%%

\subsection{Proof of Lemma \ref{MainTheoremLemma}}
\label{K-CAsynchNoiseLemmaProof}

Extend our notation to continuous time by defining $M(t)$ as the $K \times K$ diagonal matrix of the $(u_1(t),\dots,u_K(t))$ terms and recall that $\tilde{M}(t)$ is the $K \times K$ diagonal matrix of the $(v_1(t),\dots,v_K(t))$ terms.

	Let $h(\cdot)$ be a bounded continuous function on $[\bar{\tau}_n,\bar{\tau}_k]$ such that $h(\bar{\tau}_j) = f_j$ for all $j =n,\ldots,k$. This will mean that $h(\cdot) \in L^2([0,T])$. Throughout this paper we consider the continuous interval $[0,\infty)$ divided into segments of length $\bar{\alpha}_n$, and hence we can approximate the sum from Lemma \ref{MainTheoremLemma} as an integral,
	
	\begin{equation*}
		\bigg\| \sum_{i=n}^{k-1} \bar{\alpha}_{i+1} f_i \big( M_{i+1} - \tilde{M}_{i+1} \big) \bigg\| = \left\| \int_{\bar{\tau}_n}^{\bar{\tau}_{k-1}} h(\bar{\tau}_{\bar{m}(t)}) \big( M_{\bar{m}(t)+1} - \tilde{M}_{\bar{m}(t)+1} \big) \d t \right\|.
	\end{equation*}
	Now we note that $u_i(t)$ and $v_i(t)$ are extensions of $\{\mu_n(i)\}_{n \in \mathbb{N}}$ to continuous time which are constant on intervals $[\bar{\tau}_n,\bar{\tau}_{n+1})$, and that $M(t)$ and $\tilde{M}(t)$ are just matrices containing $u_i(t)$ and $v_i(t)$ respectively. This will mean that $M(t) = M_{\bar{m}(t) + 1}$ and $\tilde{M}(t) = \tilde{M}_{\bar{m}(t) + 1}$ and hence we have,

	\begin{align}
		& \left\| \int_{\bar{\tau}_n}^{\bar{\tau}_{k-1}} h(\bar{\tau}_{\bar{m}(t)}) \big( M_{\bar{m}(t)+1} - \tilde{M}_{\bar{m}(t)+1} \big) \d t \right\| \label{AysnchNoise} \\
		& \quad\quad\quad = \left\| \int_{\bar{\tau}_n}^{\bar{\tau}_{k-1}} h(\bar{\tau}_{\bar{m}(t)}) \big( M(t)- \tilde{M}(t) \big) \d t \right\| \notag
	\end{align}
	In order to prove the convergence to zero of this term we look to use \eqref{mtilderesult}, but a further expansion is required so that \eqref{AysnchNoise} is in the correct form.
	
	\begin{align}
		& \left\| \int_{\bar{\tau}_n}^{\bar{\tau}_{k-1}} h(\bar{\tau}_{\bar{m}(t)}) \big( M(t)- \tilde{M}(t) \big) \d t \right\| \notag \\
		& \quad\quad\quad \leq \left\| \int_{\bar{\tau}_n}^{\bar{\tau}_{k-1}} h(t) \big( M(t)- \tilde{M}(t) \big) \d t \right\| \label{Riemannconvergence1} \\
		& \quad\quad\quad\quad + \left\| \int_{\bar{\tau}_n}^{\bar{\tau}_{k-1}} M(t) \big[h(\bar{\tau}_{\bar{m}(t)}) - h(t)\big] \d t \right\| \label{Riemannconvergence2} \\
		& \quad\quad\quad\quad + \left\| \int_{\bar{\tau}_n}^{\bar{\tau}_{k-1}} \tilde{M}(t) \big[h(t) - h(\bar{\tau}_{\bar{m}(t)})\big] \d t \right\|. \label{Riemannconvergence3}
	\end{align}
	Now, to prove the claim of Lemma \ref{MainTheoremLemma} we will show that each of these terms will converge to zero. Firstly, \eqref{Riemannconvergence1} converges to zero almost surely by \eqref{mtilderesult}. 
	
	Both \eqref{Riemannconvergence2} and \eqref{Riemannconvergence3} are dealt with using the same technique, which we demonstrate for \eqref{Riemannconvergence3}.
			
	\begin{align*}
		& \left\| \int_{\bar{\tau}_n}^{\bar{\tau}_{k-1}} \tilde{M}(t) \big[h(t) - h(\bar{\tau}_{\bar{m}(t)})\big] \d t \right\| \\
		& \quad\quad\quad \leq \left\| \int_{\bar{\tau}_n}^{\bar{\tau}_{k-1}} \tilde{M}(t) h(t) \d t - \sum_{i=n}^{k-1} \bar{\alpha}_{i+1} h(\bar{\tau}_i)\tilde{M}(\bar{\tau}_i)\right\| \\
		& \quad\quad\quad\quad + \left\| \sum_{i=n}^{k-1} \bar{\alpha}_{i+1} h(\bar{\tau}_i)\tilde{M}(\bar{\tau}_i) - \int_{\bar{\tau}_n}^{\bar{\tau}_{k-1}} \tilde{M}(t) h(\bar{\tau}_{\bar{m}(t)}) \d t \right\|,
	\end{align*}
	in both of these terms the sum is a Riemann approximation to the integral using the left hand points of the partition $[\bar{\tau}_n,\bar{\tau}_{n+1}]$ for all $n$. As $n \rightarrow \infty$ the width of the partition tends to zero and hence the above terms tend to zero. This proves that \eqref{Riemannconvergence2} and \eqref{Riemannconvergence3} will converge to zero.\hfill $\square$

%%%%%%%%%%%%%%%%%%%%%%%%%%%%
%%%%%%%%%%%%%%%%%%%%%%%%%%%%
%%%%%%%%%%%%%%%%%%%%%%%%%%%%
%%%%%%%%%%%%%%%%%%%%%%%%%%%%
%%%%%%%%%%%%%%%%%%%%%%%%%%%%
%%%%%%%%%%%%%%%%%%%%%%%%%%%%

\bibliographystyle{plain}
\bibliography{References}

\end{document}